\documentclass[11pt]{article}

\usepackage{amsfonts,amssymb,amsthm, amsmath, mathtools}
\usepackage{dsfont}
\usepackage{graphicx}
\usepackage{fullpage}
\usepackage{changepage}
\usepackage{enumitem}
\usepackage{enumerate}
\usepackage{algorithm}
\usepackage{titling}
\usepackage{algorithmic}
\usepackage{scalerel}
\usepackage{accents}
\usepackage{layout}
\usepackage{printlen}
\usepackage{booktabs}
\usepackage{float}
\usepackage{xspace}
\usepackage[T1]{fontenc}
\usepackage{lmodern}
\usepackage{aliascnt}
\usepackage{comment}

\usepackage{xcolor}
\usepackage{graphicx}
\usepackage{enumitem}
\usepackage{subcaption}
\usepackage[export]{adjustbox} 

\usepackage{optidef}

\usepackage{nicefrac,xfrac}

\usepackage[toc,page]{appendix}

\usepackage{makecell}
\usepackage{tablefootnote}
\usepackage{lipsum}
\usepackage{caption}
\usepackage{subcaption}
\usepackage{hyperref}
\usepackage{mfirstuc}
\usepackage{multirow}
\usepackage{array}

\usepackage{natbib}
\usepackage{csquotes}
\usepackage{comment}

\hypersetup{hidelinks}

\usepackage{bbm}

\usepackage[noabbrev]{cleveref}
\crefname{claim}{claim}{claims}
\crefname{enumi}{property}{properties}
\crefname{adversary}{adversary}{adversaries}

\usepackage{tikz}
\usepackage{pgfplots}
\pgfplotsset{width=10cm,compat=1.9}
\definecolor{darkpastelgreen}{rgb}{0.01, 0.75, 0.24}
\definecolor{darkorchid}{rgb}{0.6, 0.2, 0.8}
\definecolor{blue(pigment)}{rgb}{0.2, 0.2, 0.6}

\Crefname{algorithm}{Algorithm}{Algorithms}
\allowdisplaybreaks

\usetikzlibrary{arrows, decorations.markings}
\tikzstyle{vecArrow} = [thick, decoration={markings,mark=at position
   1 with {\arrow[semithick]{open triangle 60}}},
   double distance=1.4pt, shorten >= 5.5pt,
   preaction = {decorate},
   postaction = {draw,line width=1.4pt, white,shorten >= 4.5pt}]
\tikzstyle{innerWhite} = [semithick, white,line width=1.4pt, shorten >= 4.5pt]

\allowdisplaybreaks
\theoremstyle{plain}
\newtheorem{theorem}{Theorem}[section]

\newtheorem*{theorem*}{Meta-Theorem}

\newaliascnt{proposition}{theorem}
\newtheorem{proposition}[proposition]{Proposition}
\aliascntresetthe{proposition}

\newaliascnt{definition}{theorem}
\newtheorem{definition}[definition]{Definition}
\aliascntresetthe{definition}

\newaliascnt{corollary}{theorem}
\newtheorem{corollary}[corollary]{Corollary}
\aliascntresetthe{corollary}

\newaliascnt{lemma}{theorem}
\newtheorem{lemma}[lemma]{Lemma}
\aliascntresetthe{lemma}

\newtheorem{example}[definition]{Example}
\newtheorem{remark}[definition]{Remark}

\theoremstyle{plain}

\usepackage{xfrac}
\usepackage{thm-restate}

\newcommand{\xhdr}[1]{\vspace{2mm} \noindent{\bf #1}}

\newcommand{\val}{v}

\newcommand{\del}{\delta}

\newcommand{\set}[1]{\left\{#1\right\}}

\newcommand{\norm}[1]{\left\lVert#1\right\rVert}
\newcommand{\ip}[2]{\left\langle #1, #2 \right\rangle}

\newcommand{\tO}{\widetilde{O}}
\DeclareMathOperator{\TV}{TV}
\DeclareMathOperator{\width}{width}

\newcommand{\TT}{T}

\newcommand{\cR}{\mathcal{R}}
\newcommand{\cRT}{\cR_{\TT}}

\newcommand{\Indic}[1]{\mathcal{I}\{#1\}}

\newcommand{\indicator}[1]{\mathbb{I}\!\set{#1}}

\newcommand{\round}{t}              \newcommand{\price}{p}             \newcommand{\hparam}{\theta}       \newcommand{\uvec}{u}              \newcommand{\batch}{B}             \newcommand{\bReg}{R}              \newcommand{\phase}{i}
\newcommand{\Tphase}{r}
\newcommand{\pimin}{p_\phase^{\min}}
\newcommand{\pimax}{p_\phase^{\max}}
\newcommand{\punderi}{\underline{p}_\phase}

\newcommand{\Wsize}{a}

\newcommand{\Sset}{S}
\newcommand{\SExtendedset}{S^{+}}

\newcommand{\Hset}{H}
\newcommand{\HExtendedset}{H^{+}}
\newcommand{\fb}{f}

\newcommand{\dimension}{d}
\newcommand{\ALG}{\textsc{Alg}}
\newcommand{\TVUB}{V}
\newcommand{\hparams}{\{\hparam_t\}_{t\in[T]}}
\newcommand{\Reg}{\textsc{Regret}}

\newcommand{\TVpB}{\Delta}

\providecommand{\Rspace}{\mathbb{R}}
\providecommand{\Kset}{K}
\providecommand{\Eset}{\mathcal{E}}
\providecommand{\Hset}{H}
\providecommand{\xvec}{x}
\providecommand{\ellmat}{A}
\providecommand{\ellcenter}{c}
\providecommand{\midprice}{\widetilde{\price}}
\providecommand{\LEP}{\operatorname{LEP}}
\providecommand{\shallowpricing}{\textsc{ShallowPricing}}

\providecommand{\ctxnorm}{q}

\providecommand{\scalelevels}{L}

\providecommand{\batchcount}{m}

\providecommand{\Nspace}{\mathbb{N}}
\providecommand{\batchend}{\tau}
\providecommand{\explorescales}{\mathcal{X}}
\providecommand{\virtualfb}{\widetilde{\fb}}
\providecommand{\shallowupdate}{\operatorname{ShallowUpdate}}

\newcommand{\reals}{\mathbb{R}}

\newcommand{\prob}[2][]{\text{Pr}\ifthenelse{\not\equal{}{#1}}{_{#1}}{}\!\left[{\def\givenn{\middle|}#2}\right]}
\newcommand{\expect}[2][]{\mathbb{E}\ifthenelse{\not\equal{}{#1}}{_{#1}}{}\!\left[{\def\givenn{\middle|}#2}\right]}

\newcommand{\tparen}{\big}
\newcommand{\tprob}[2][]{\text{Pr}\ifthenelse{\not\equal{}{#1}}{_{#1}}{}\tparen[{\def\given{\tparen|}#2}\tparen]}
\newcommand{\texpect}[2][]{\mathbb{E}\ifthenelse{\not\equal{}{#1}}{_{#1}}{}\tparen[{\def\given{\tparen|}#2}\tparen]}

\newcommand{\sprob}[2][]{\text{Pr}\ifthenelse{\not\equal{}{#1}}{_{#1}}{}[#2]}
\newcommand{\sexpect}[2][]{\mathbb{E}\ifthenelse{\not\equal{}{#1}}{_{#1}}{}[#2]}

\title{Optimal Contextual Pricing in Non-Stationary Environments}

\author{
Yiding Feng\thanks{Department of Industrial Engineering
and Decision Analytics, Hong Kong University
of Science and Technology.
Email: \texttt{ydfeng@ust.hk}.}
\and
Alireza Rezaeimoghadam\thanks{School of Computer and Communication Sciences,
EPFL. Email: \texttt{alireza.rezaeimoghadam@epfl.ch}. Part of this work was conducted while
Alireza Rezaeimoghadam was a research intern at
the Hong Kong University of Science and Technology
(HKUST), hosted by Yiding Feng.}}

\date{}

\begin{document}

\maketitle

\begin{abstract}
We study contextual dynamic pricing in a non-stationary
environment. In each round, a seller observes a buyer's
context, posts a price, and receives only binary purchase
feedback. The buyer's value is linear in the context and
an unobserved feature vector that may change over time,
subject to a known total variation budget $\TVUB$.
We measure dynamic regret against the first-best benchmark
that extracts each buyer's full value.

For identical contexts, we develop a deterministic algorithm
that combines binary-search exploration, lower-envelope
pricing, and two complementary restart mechanisms.
It achieves regret
$O((\sqrt{\TT\TVUB}+1)\log\TT)$.
With heterogeneous contexts, a new geometric difficulty
arises: small feature changes can leave a relaxed feasible
set narrow but far from the current feature vector.
To address this difficulty, we introduce a multiscale
shallow-pricing algorithm that maintains relaxed feasible
sets and ellipsoidal approximations at several variation
scales.
Larger scales provide reliable localization certificates,
while the smallest active scale determines the exploitation
price.
The algorithm achieves dynamic regret
$O((\dimension^2+\dimension^{3/2}\sqrt{\TT\TVUB})
\log^2\TT)$.
Both guarantees match the identical-context lower bound
$\Omega(\sqrt{\TT\TVUB}+\log\log\TT)$ up to logarithmic
factors. \end{abstract}

\thispagestyle{empty}
\newpage

\section{Introduction}
\label{sec:introduction}

Online sellers repeatedly make pricing decisions for buyers with
different observable characteristics and changing willingness to pay.
The economic scale of these decisions is substantial: U.S. consumers
spent \$257.8 billion online during November and December 2025,
including \$44.2 billion during the five-day Cyber Week
\citep{Adobe2026HolidayShopping}.
Observable context can help a seller tailor prices to individual
transactions, but the relationship between context and willingness
to pay may change with seasonality, preferences, and market conditions.
Learning this relationship is particularly challenging because
pricing feedback is censored: a purchase or rejection reveals only
whether the buyer's value exceeds the posted price.
An effective pricing algorithm must therefore learn from limited
feedback while adapting to a changing environment.

We study contextual dynamic pricing over $\TT$ rounds.
In each round, a new buyer arrives with an observable
$\dimension$-dimensional context and an unobserved
$\dimension$-dimensional feature vector.
The buyer's value is the inner product of these two vectors.
Contexts are nonnegative unit vectors, and hidden features are
nonnegative vectors with Euclidean norm at most one, so values
lie between zero and one.
After observing the context, the seller posts a price and observes
only whether the buyer purchases.
Contexts may vary arbitrarily across rounds.
The hidden features may also change, subject to a known upper bound
$\TVUB$ on the sum of their Euclidean changes between consecutive
rounds.
Our objective is to minimize \emph{dynamic regret} relative to the
\emph{first-best benchmark}: the total revenue of an oracle that
knows each buyer's value and posts that value as the price.
This benchmark extracts the full surplus in every round and is
stronger than comparison with the best single fixed price in hindsight \citep{KL-03}.

Our problem builds on the study of online posted-price learning
initiated by \citet{KL-03} and connects two lines of dynamic pricing
research.
Stationary contextual pricing studies heterogeneous contexts with
a fixed hidden parameter \citep{LS-18,CLP-20,LLS-21}, whereas
\citet{LemeEtAl21MovingTarget} study non-contextual pricing with
changing values under bounds on their per-round changes.
We consider heterogeneous contexts together with changing hidden
features under a known aggregate variation budget. Non-stationary contextual pricing has also been studied under
stochastic valuation noise and demand models
\citep{Javanmard2017,BabyEtAl2023}; our focus is the noiseless
linear-valuation model and regret against the first-best benchmark. 
(A comprehensive discussion on the further related work can be found in \Cref{apx:related-work}.)
This leads to the following question:
\begin{quote}
\emph{What is the optimal dynamic regret in this setting,
and what additional difficulties arise from the interaction
between contextual heterogeneity and non-stationarity?}
\end{quote}

\subsection{Our Results and Techniques}
We develop a deterministic contextual pricing algorithm with
dynamic regret
$O\!\left(
    (\dimension^2+
    \dimension^{3/2}\sqrt{\TT\TVUB})\log^2\TT
\right)$.
Together with a lower bound of
$\Omega(\sqrt{\TT\TVUB}+\log\log\TT)$,
this establishes optimality up to logarithmic factors in $\TT$
for fixed $\dimension$ among deterministic algorithms.
The guarantee holds for every variation budget
$\TVUB\in[0,\TT]$, including the stationary case.
Our algorithm combines periodic and inconsistency restarts
with a multiscale extension of shallow pricing.

We first develop the restart mechanism in the identical-context
environment (\Cref{sec:identical-context}).
In this setting, binary search followed by lower-envelope pricing
provides a simple exploration-then-exploitation procedure.
Non-stationarity creates two distinct difficulties:
an increase in values can leave an accepted price persistently
too low, whereas even a small decrease can turn a previously
accepted price into a rejection and cause a large revenue loss.
We address these difficulties through periodic restarts and
immediate restarts after rejected exploitation prices.
A price margin controls the frequency of the latter by ensuring
that each such rejection requires sufficient variation.
This yields dynamic regret
$O((\sqrt{\TT\TVUB}+1)\log\TT)$.

Extending this argument to heterogeneous contexts requires
additional geometric reasoning.
A narrow interval directly localizes a scalar value, but
localization in one context direction need not provide useful
information in another.
We therefore build on the shallow-pricing procedure of
\citet{CLP-20}, which uses ellipsoidal approximations to guide
exploration.
To accommodate non-stationarity, we relax the halfspace constraints
induced by purchase feedback.
However, using a single relaxation scale introduces a new
difficulty: a relaxed feasible set can become narrow while
excluding the current hidden parameter.
Nearly parallel context directions can amplify small parameter
changes into a large localization error.
Our Example~\ref{ex:narrow-misplaced-set} shows why
the identical-context exploitation argument does not
extend directly.

Our main algorithm resolves this difficulty by maintaining
relaxed feasible sets and enclosing ellipsoids at multiple
dyadic scales (\Cref{sec:contextual}).
It explores whenever an active scale indicates substantial
uncertainty and otherwise posts the lower-envelope price
of the smallest active feasible set.
The key observation is that a scale at least as large as
the within-batch parameter diameter retains every hidden
parameter in that batch.
Its ellipsoid therefore provides a reliable certificate
of localization.
Nesting of the feasible sets then controls the regret
of accepted exploitation prices, even when the smallest
active set excludes the current parameter.
Rejected exploitation prices trigger immediate restarts.
Combining this multiscale argument with geometric bounds
on exploration and variation-based bounds on the number
of restarts yields our dynamic regret guarantee.

\subsection{Further Related Work}
\label{apx:related-work}

We discuss additional connections to online pricing, contextual search,
and learning in non-stationary environments. The main distinctions
concern the valuation model, the measure of non-stationarity, the
information available to the seller, and the revenue benchmark.

\xhdr{Online posted pricing and demand learning.}
\citet{KL-03} study online posted-price auctions under identical,
random, and worst-case valuations. For an unknown identical value,
they establish the optimal $\Theta(\log\log\TT)$ regret. In this
special case, the best fixed price and the first-best benchmark
coincide. For arbitrary valuation sequences, their comparison is
instead with the best fixed price in hindsight. A complementary
operations research literature studies the exploration--exploitation
tradeoff in learning an unknown demand function. For example,
\citet{BesbesZeevi2009} develop pricing policies and regret bounds
relative to a seller who knows the demand function, in a setting
with finite inventory. Our model has no intertemporal inventory
constraint and compares revenue with the sum of buyers' realized
values. Thus, differences in regret rates across these settings
reflect differences in both information and benchmarks.

\xhdr{Stationary contextual pricing and geometric search.}
\citet{CLP-20} study feature-based pricing with an unknown linear
valuation model and develop ellipsoid-based algorithms with regret
polynomial in $\dimension$ and logarithmic in $\TT$.
They also introduce the robust \shallowpricing{} procedure, which
relaxes feedback constraints to accommodate small valuation noise.
Our algorithms use this geometric primitive.
Related geometric methods appear in multidimensional binary search
\citep{LobelEtAl2018} and contextual search via intrinsic volumes
\citep{LS-18}.
\citet{LLS-21} obtain an
$O(\dimension\log\log\TT+\dimension\log\dimension)$ regret bound
for stationary contextual pricing under their normalization and
develop extensions to more general contextual search problems.
These works exploit a fixed underlying parameter or hypothesis.
In our setting, feedback constraints collected at different times
may correspond to different hidden parameters. Consequently,
shrinking an uncertainty set need not localize the current parameter. Extension models have been studied in \citet{KLPS-21,LNOPZ-25,GBCCP-25,FMPZ-26a,FMPZ-26b,CFGLPR-26}.
Our contribution is to control this loss of consistency while
retaining the exploration guarantees of geometric search.
Multiscale methods themselves also appear in prior contextual
search algorithms, including \citet{LLS-21}; here, the scales
represent different tolerances for parameter drift. 

\xhdr{Contextual pricing with stochastic valuation noise.}
A substantial literature assumes that valuations consist of a
linear contextual component and an additional random shock.
\citet{JavanmardNazerzadeh2019} study high-dimensional pricing
with sparse parameters and develop regularized likelihood methods.
\citet{XuWang2021} establish logarithmic regret guarantees with
a known noise distribution, including for adversarial contexts.
When the noise distribution is unknown,
\citet{XuWang2022} distinguish learning a linear pricing policy
from learning a linear valuation model with agnostic noise, while
\citet{LuoEtAl2022} combine parameter exploration with a UCB
procedure for learning the noise distribution and maximizing revenue.
More recently, \citet{XuWang2026} study unknown noise distributions
that induce discontinuous demand, allowing jumps and atoms under
stochastic, well-conditioned contexts.
These models introduce statistical uncertainty beyond the unknown
linear parameter. Their distribution-aware oracle generally knows
the demand model but does not observe each buyer's realized noise
before pricing. Our noiseless model instead permits an oracle
that knows the current parameter and context to extract the
buyer's full value. Accordingly, the associated regret guarantees
address different benchmarks and cannot be compared solely by
their dependence on $\TT$.

\xhdr{Non-contextual pricing against a moving value.}
The closest predecessor to our identical-context setting is
\citet{LemeEtAl21MovingTarget}, who study changing scalar values
with binary purchase feedback and the same first-best benchmark.
They analyze adversarial evolution under pointwise bounds on
consecutive value changes, as well as stochastic evolution.
Their adversarial revenue guarantees cover a fixed changing rate
and unknown, non-increasing sequences of changing-rate bounds.
The monotonicity condition concerns these bounds, not the values
themselves. Their algorithms combine localization with randomized
checking rounds to detect discrepancies between the current
value and the maintained interval.
Our formulation assumes a known aggregate variation budget
$\TVUB$ and permits this variation to be distributed arbitrarily
over time. Our identical-context algorithm uses periodic
restarts and deterministic inconsistency restarts, together with
a price margin that controls their frequency. It achieves
$O((\sqrt{\TT\TVUB}+1)\log\TT)$ regret and supplies the restart
mechanism used in our contextual algorithm. The two formulations
also differ in prior information: knowing the aggregate budget
is an assumption of our algorithm, whereas their unknown-rate
algorithms adapt without this information.

\xhdr{Non-stationary pricing with stochastic demand.}
Non-stationary demand learning has also been studied beyond the
moving-value model. \citet{KeskinZeevi2017} analyze pricing with
an unknown changing demand model and characterize regret under
variation constraints. In the contextual setting,
\citet{Javanmard2017} studies a linear valuation model with random
shocks and time-varying coefficients, using projected stochastic
gradient methods. Their guarantees cover both adversarial and
stochastic context sequences and explicitly account for changes
in the coefficients.
\citet{BabyEtAl2023} study non-stationary contextual pricing with
safety constraints. With known strictly log-concave market noise,
they derive dynamic regret guarantees through proper online
learning with exp-concave losses and adapt to an unknown parameter
variation budget. Their benchmark is the sequence of optimal
expected-revenue prices under the current demand model.
These papers already address contextual heterogeneity together
with non-stationarity. Our focus is the noiseless linear-valuation
model, an aggregate Euclidean variation budget, and regret against
full surplus, without assuming the smooth demand probabilities
used in their statistical analyses.

Change detection provides another approach to non-stationary
pricing. \citet{ZhaoEtAl2026} develop a policy based on online
change-point detection for high-dimensional generalized linear
demand models, with guarantees that account for sparsity and
the number of changes. \citet{JiangZhao2026} develop multiscale
change-point detection that adapts to both piecewise stationarity
and a design-adjusted parameter variation budget, without knowing
these measures in advance. Their model uses stochastic contexts
and generalized linear demand, and their oracle maximizes
expected revenue. Their scales organize statistical detection
over time. In our algorithm, scales index simultaneous
relaxations of the same geometric feedback constraints and
determine whether localization is reliable enough for exploitation.
This distinction also clarifies that our contribution is not the
general use of restarts or multiple scales in non-stationary pricing.

\xhdr{Non-stationary bandits and restart methods.}
\citet{BGZ-14} study stochastic multi-armed bandits under
a variation budget on mean rewards and establish near-optimal
dynamic regret through periodic restarts of Exp3.
\citet{CheungEtAl2019} develop sliding-window methods and a
bandit-over-bandit approach for non-stationary linear bandits.
\citet{ChenEtAl2019} give an adaptive contextual bandit algorithm
that uses replay phases to detect changes without prior knowledge
of the amount of non-stationarity.
These works motivate forgetting outdated observations and
combining exploration with change detection. However, their
variation measures and reward structures differ from ours.
Even with a fixed context, an arbitrarily small decrease in
the hidden parameter can change a purchase into a rejection
at a fixed positive price, causing a constant revenue drop.
Thus, a small parameter variation budget does not imply a
comparably small uniform variation budget for revenue functions.
Generic reward-variation guarantees therefore do not directly
yield our regret bound. Our analysis instead uses the geometry
of purchase feedback and the asymmetric consequences of
underpricing and overpricing.

\xhdr{Dynamic regret and online convex optimization.}
Dynamic regret against a changing comparator is a classical
objective in online optimization.
\citet{Zinkevich2003} analyzes online convex optimization with
moving comparators, and \citet{BesbesEtAl2015} study stochastic
optimization under variation budgets on the objective functions.
\citet{BabyWang2021} establish optimal dynamic regret guarantees
for exp-concave online learning when improper predictions are
allowed. The pricing results of \citet{BabyEtAl2023} connect
this literature to constrained estimation and stochastic demand.
Our revenue loss is discontinuous and is not a convex function
of the posted price, so these guarantees do not apply directly.
Nevertheless, the distinction between comparator movement,
objective variation, and parameter variation is relevant here:
our budget constrains hidden features, while contexts may vary
arbitrarily. In particular, even when $\TVUB=0$, both buyer
values and first-best prices may change across rounds.

\xhdr{Corruption-robust contextual search.}
\citet{KrishnamurthyEtAl2023} study contextual search when some
responses are adversarially corrupted and design algorithms
whose performance degrades with the number of corruptions.
\citet{LemeEtAl2022} develop density-based methods for
adversarially noisy contextual search, obtaining guarantees
for symmetric and tolerance-based losses.
These works share the concern that intersecting all observed
constraints can eliminate the correct hypothesis.
The source of inconsistency differs in our model: feedback
is accurate for the current buyer, but the hidden parameter
itself moves. A single persistent parameter shift may make
many subsequent observations inconsistent with an earlier
parameter while contributing only once to our variation budget.
Moreover, a small symmetric prediction error does not ensure
small pricing regret, because even slight overpricing can
eliminate the entire revenue from a round.
Thus, corruption budgets and search-loss guarantees require
additional arguments before they can be used for our objective.

\section{Preliminaries}
\label{sec:problem-setup}

We introduce the contextual dynamic pricing problem in a non-stationary environment.

\xhdr{Environment.} There is a single seller who interacts with buyers over $\TT$ discrete rounds. In each round $t\in[\TT]\triangleq \{1,2,\dots,\TT\}$, a new buyer, indexed by $t$, arrives. Each buyer $t$ has a public $\dimension$-dimensional context vector $\uvec_t\in[0,1]^\dimension$ with\footnote{Throughout this paper, we use $\norm{\cdot}$ to denote the Euclidean norm.} $\norm{\uvec_t}=1$ and a private $\dimension$-dimensional feature vector $\hparam_t\ge0, \norm{\hparam_t}\le 1$. After observing the public context $\uvec_t$, the seller posts a take-it-or-leave-it price $\price_t\in[0,1]$ for a single product. Buyer $t$'s value (i.e., willingness to pay) takes the linear form $\val_t\triangleq\ip{\hparam_t}{\uvec_t}\in[0,1]$. Thus, buyer $t$ purchases the product at price $\price_t$ if and only if her value is at least the price, i.e., $\val_t\geq\price_t$. The seller observes only whether the buyer purchases the product and does not observe her value. 
(We denote the purchase feedback by
$\fb_t\triangleq\indicator{\val_t\geq\price_t}\in\{0,1\}$.)
The seller's goal is to maximize her revenue, i.e., the total payments from all $\TT$ buyers. We assume $\TT\geq2$.
No stochastic assumptions are imposed on the context or
feature sequences.

\xhdr{Dynamic regret minimization.} We evaluate an online learning algorithm by its regret relative to the \emph{first-best benchmark} (i.e., full surplus), $\sum_{t\in[\TT]}\val_t=\sum_{t\in[\TT]}\ip{\hparam_t}{\uvec_t}$. This benchmark is achieved by posting the full-surplus-extraction price $\price_t^*=\val_t$ for each buyer $t$. Specifically, given a sequence of $\TT$ buyers described by $\{\uvec_t,\hparam_t\}_{t\in[\TT]}$, the dynamic regret of an algorithm $\ALG$ is
\begin{align*}
    \Reg[\ALG]
    \triangleq
    \sum\nolimits_{t\in[\TT]} \ip{\hparam_t}{\uvec_t}
    -
    \sum\nolimits_{t\in[\TT]} \price_t \cdot
    \indicator{\ip{\hparam_t}{\uvec_t}\geq\price_t},
\end{align*}
where $\{\price_t\}_{t\in[\TT]}$ are the prices posted by $\ALG$.

\xhdr{Non-stationarity upper bound.}
We measure the \emph{non-stationarity} of the unknown feature sequence $\hparams=(\hparam_t)_{t\in[\TT]}$ by its total variation,
$\TV(\hparams)\triangleq\sum_{t\in[\TT-1]}\norm{\hparam_{t+1}-\hparam_t}$.
Although the feature sequence $\hparams$ remains unobserved by the seller throughout the $\TT$ rounds, we assume that the seller knows an upper bound $\TVUB$ on its total variation in advance, i.e., $\TV(\hparams)\leq\TVUB$. We focus on $\TVUB\in[0,\TT]$; larger variation budgets
admit the trivial regret guarantee $\TT$.
In the special case of zero non-stationarity, i.e., $\TVUB=0$, our problem reduces to the classical stationary contextual dynamic pricing problem studied in \citet{KL-03,LS-18,CLP-20,LLS-21}.

\section{Warm-Up: Optimal Dynamic Pricing for Identical Contexts}
\label{sec:identical-context}

We begin with the identical-context environment, where all buyers have the same public context, denoted by $\uvec$. This section serves as a warm-up for the more general environment with heterogeneous contexts: its algorithmic ingredients and regret analysis closely parallel, and lay the groundwork for, the design and analysis of our main algorithm in \Cref{sec:contextual}.

Although the public contexts are identical, the private features $\{\hparam_t\}_{t\in[\TT]}$ may still vary across buyers. We can therefore view this setting as a pricing problem without public contexts, in which buyers have potentially different private values $\{\val_t\}_{t\in[\TT]}$.
For ease of presentation, we omit $\{\uvec_t,\hparam_t\}_{t\in[\TT]}$ in this section and work directly with the value sequence $\{\val_t\}_{t\in[\TT]}$. Under the normalization in our model, $\val_t\in[0,1]$ for every buyer $t$. Moreover, since identical context $\uvec$ is unit-length (i.e., $\norm{\uvec}=1$), we have $
    \sum_{t\in[\TT - 1]}|\val_{t+1}-\val_t|
    \leq
    \sum_{t\in[\TT - 1]}\norm{\hparam_{t+1}-\hparam_t}
    \leq \TVUB$.
We consider non-stationarity upper bounds $\TVUB\in[0,\TT]$.

Below, we present the algorithm for this environment and its dynamic regret guarantee. We then conclude the section with a sketch of the regret analysis.

\xhdr{Algorithm overview.}
For stationary online learning problems, a standard approach is
\emph{explore-then-exploit} (ETE), also commonly called
\emph{explore-then-commit} in the bandit literature
\citep[Chapter~6]{LS-20}. When all buyers have the same value, this approach can be implemented by first conducting binary search for $\lceil\log_2\TT\rceil$ rounds to locate the value within an interval of width at most $1/\TT$. This exploration phase incurs at most $\lceil\log_2\TT\rceil$ regret. During the subsequent exploitation phase, posting the lower endpoint of this interval in every remaining round incurs at most $O(1)$ additional regret.\footnote{In the setting with an unknown identical value, the optimal regret bound of $O(\log\log\TT)$ can be achieved by a modified ETE approach that replaces ordinary binary search with an asymmetric search procedure \citep{KL-03}.}

For non-stationary environments, a standard approach is to restart
a learning algorithm \emph{periodically}. For example, in non-stationary
multi-armed bandit problems, \citet{BGZ-14} show
that periodically restarting Exp3 achieves near-optimal dynamic
regret under a variation-budget constraint. However, the revenue structure of dynamic pricing creates an additional difficulty. An arbitrarily small decrease in a buyer's value can reduce the revenue from a previously optimal price to zero: if the seller continues posting $\price_{t+1}=\val_t$ but $\val_{t+1}<\val_t$, the buyer rejects the offer. Thus, even an arbitrarily small amount of variation can cause constant regret in every subsequent round until the next periodic restart, provided the buyer's value remains bounded away from zero.\footnote{For comparison, in a bandit problem with a fixed set of arms, an arm that was optimal at an earlier round has a current suboptimality gap at most twice the largest change in any arm's mean reward since that round.}

Motivated by this difficulty, our algorithm supplements periodic restarting with an \emph{inconsistency restart}. Informally, if a buyer rejects a price that the algorithm posts during exploitation, the algorithm immediately restarts ETE. However, inconsistency restarting alone cannot guarantee sublinear regret either. To see this, consider a sequence in which all buyers during exploration have the same value. Immediately after exploration ends, the value increases by $\TVUB$ and remains at this higher level for the rest of the horizon, where $\TVUB>0$ is chosen so that all values remain in $[0,1]$. No inconsistency restart occurs because buyers continue accepting the exploitation price. The regret in each exploitation round is at least $\TVUB$, leading to total regret of $\Omega(\TVUB\TT)$ when exploration takes $O(\log\TT)$ rounds. This regret is linear in $\TT$ for any fixed admissible $\TVUB>0$.

Our algorithm therefore combines periodic and inconsistency restarts. We call the consecutive rounds from initialization or a restart until the next restart a \emph{batch}. Periodic restarting is controlled by a \emph{window-size parameter $\Wsize\in[\TT]$}, which specifies the maximum number of rounds in a batch. This parameter balances the regret caused by non-stationarity during exploitation against the exploration cost of frequent periodic restarts.

Inconsistency restarting is controlled by a \emph{price-margin parameter $\del\in(0,1]$}. This parameter determines both the target interval width during exploration and the price posted during exploitation. Specifically, once binary search produces an interval of width at most $\del$, the algorithm posts its lower endpoint minus $\del$, clipped at zero, during the subsequent exploitation phase. We call this price the \emph{lower-envelope price} (LEP). Intuitively, $\del$ balances the regret from conservative pricing during exploitation against the exploration cost of frequent inconsistency restarts. The formal algorithm is given in \Cref{alg:1D-search}.

\begin{algorithm}[t]
\footnotesize
\caption{Pricing for the identical-context environment}
\label{alg:1D-search}
\begin{algorithmic}[1]
\STATE \textbf{Input:}
$\TT\in\Nspace$, $\TVUB\in[0,\TT]$
\STATE \textbf{Parameters:}
$\del\gets\max\{1/\TT,\sqrt{\TVUB/\TT}\}$,
$\Wsize\gets\min\{\TT,\lceil\sqrt{\TT/\TVUB}\rceil\}$
\hfill$\triangleright$ set $\Wsize \gets T$ when $\TVUB = 0$
\STATE $t\gets1$, $i\gets0$
\WHILE{$t\leq\TT$}
    \STATE $i\gets i+1$,
    $\batchend\gets\min\{\TT,t+\Wsize-1\}$,
    $[\pimin,\pimax]\gets[0,1]$
    \WHILE{$t\leq\batchend$ and $\pimax-\pimin>\del$}
        \STATE Post $\price_t\gets(\pimin+\pimax)/2$;
        observe $\fb_t$
        \STATE \textbf{if} $\fb_t=1$,
        set $\pimin\gets\price_t$;
        \textbf{else} set $\pimax\gets\price_t$
        \STATE $t\gets t+1$
    \ENDWHILE
    \STATE $\punderi\gets\max\{0,\pimin-\del\}$
    \WHILE{$t\leq\batchend$}
        \STATE Post $\price_t\gets\punderi$;
        observe $\fb_t$; set $t\gets t+1$
        \STATE \textbf{if} $\fb_{t-1}=0$,
        \textbf{break}
        \hfill$\triangleright$ inconsistency restart
    \ENDWHILE
\ENDWHILE
\end{algorithmic}
\end{algorithm}

\xhdr{Regret guarantee and its analysis.}
We now state the regret guarantee in \Cref{thm:1d-search}.

\begin{theorem}
    \label{thm:1d-search}
    In the identical-context environment, given any non-stationarity upper bound $\TVUB\in[0,\TT]$, set $\Wsize = \min\{T,\lceil\sqrt{T/\TVUB}\rceil\}$ and $\del = \max\{1/\TT,\sqrt{\TVUB/\TT}\}$\footnote{with the convention $\Wsize=\TT$ when $\TVUB=0.$}.
    Then the dynamic regret of \Cref{alg:1D-search} is at most
    $O(\sqrt{\TT\TVUB}\log\TT+\log\TT)$.
Moreover, for every $\TVUB\in[0,\TT]$ and every deterministic online algorithm,
    there exists an admissible instance with total variation at most
    $\TVUB$ on which the algorithm incurs dynamic regret of at least
    $\Omega(\sqrt{\TT\TVUB}+\log\log\TT)$.
\end{theorem}
Combining the upper and lower bounds in
\Cref{thm:1d-search}, the regret guarantee of
\Cref{alg:1D-search} is optimal among deterministic online
algorithms, up to logarithmic factors in $\TT$, for every
$\TVUB\in[0,\TT]$. The lower bound combines the
variation-dependent term $\Omega(\sqrt{\TT\TVUB})$, obtained
by adapting the construction in
\citet[Theorem~4.2]{LemeEtAl21MovingTarget}, with the stationary
lower bound $\Omega(\log\log\TT)$ from \citet{KL-03}.
The former considers the stronger restriction
$|\val_{t+1}-\val_t|\leq\TVUB/\TT$ on changes between
consecutive rounds,\footnote{Under the assumption that consecutive value changes are bounded by unknown, nonincreasing rates $\epsilon_t$ satisfying
$\sum_{t\in[\TT-1]}\epsilon_t\leq\TVUB$,
\citet{LemeEtAl21MovingTarget} propose a randomized
pricing algorithm whose analysis yields expected regret of
${O}(\sqrt{\TVUB\TT}\log \TT+\sqrt{\TT} \log \TT)$.
Their algorithm organizes learning into phases and uses randomly
selected checking rounds to detect inconsistencies.
Our inconsistency test is deterministic and also serves as a
building block of our algorithm for heterogeneous contexts.} while the latter applies when all buyers
have the same value.
We conclude this section with a sketch of the upper-bound proof. The full proof appears in \Cref{app:noncontextual-upper}.

To bound the dynamic regret of \Cref{alg:1D-search}, we bound the regret from exploration and exploitation separately. During exploration, binary search takes at most $\lceil\log_2(1/\del)\rceil$ rounds per batch. Since each round incurs at most one unit of regret, the total exploration regret is at most
$m\lceil\log_2(1/\del)\rceil$, where $m$ is the number of batches. Periodic restarts contribute at most $\TT/\Wsize$ completed batches, and each inconsistency restart requires at least $\del$ variation within its batch. Hence, $m\leq 1+{\TT}/{\Wsize}+{\TVUB}/{\del}
    =O(\sqrt{\TT\TVUB}+1)$.
Since $\del\geq 1/\TT$, the total exploration regret is therefore at most
$
    O(
        (\sqrt{\TT\TVUB}+1)\log\TT
    )
    =
    O(\sqrt{\TT\TVUB}\log\TT+\log\TT)$.

To analyze exploitation, let $\batch_i\subseteq[\TT]$ denote the rounds in batch $i$, and define its value range as $\TVpB_i\triangleq\max_{s,t\in\batch_i}|\val_s-\val_t|$. These within-batch value ranges satisfy $\sum_{i\in[m]}\TVpB_i\leq\TVUB$. Each accepted exploitation round in batch $i$ incurs regret of at most $2\del+\TVpB_i$. Each rejected exploitation round incurs at most one unit of regret and immediately ends the batch, so there is at most one such round per batch. Since each batch contains at most $\Wsize$ rounds, the total exploitation regret is at most $2\TT\del+\sum_{i\in[m]}|\batch_i|\TVpB_i+m\leq2\TT\del+\Wsize\TVUB+m=O(\sqrt{\TT\TVUB}+1)$. Combining the exploration and exploitation bounds proves the stated upper bound.

\section{Optimal Dynamic Pricing using Multi-Scale Shallow Pricing}
\label{sec:contextual}

\newcommand{\spscale}{\lambda}
\newcommand{\spscaleSet}{\Lambda}

In this section, we turn to the more general environment in which buyers have heterogeneous $\dimension$-dimensional context vectors. We present our main result: an online learning algorithm whose dynamic regret has optimal dependence on $\TT$ and the non-stationarity upper bound $\TVUB$, up to logarithmic and dimension-dependent factors. Our algorithm (\Cref{alg:multiscale-contextual}) and its regret analysis extend the approach developed for the identical-context environment in \Cref{sec:identical-context}.

Compared with the identical-context setting, our algorithm and its analysis introduce two main technical ingredients.

The first ingredient is to replace binary-search exploration with \emph{shallow pricing}, introduced by \citet{CLP-20} for stationary contextual dynamic pricing, while retaining lower-envelope pricing (LEP) for exploitation. With identical contexts, uncertainty concerns a single scalar value and can be represented by an interval. With heterogeneous contexts, uncertainty concerns a vector, and learning its projection in one context direction need not provide accurate information in another. Shallow pricing maintains an ellipsoidal approximation of the feasible region and uses its width in the current context direction to determine whether further exploration is needed. Thus, exploration and exploitation may be interleaved as different contexts arrive.

The second ingredient is to maintain \emph{multiple copies of the geometric learning procedure, each with a price-margin parameter $\spscale$ drawn from a dyadic set $\spscaleSet=\{\del,2\del,\ldots,2^L\del\}$}, whose largest scale covers the diameter of the feature domain. This multiscale construction addresses a difficulty arising from the combination of heterogeneous contexts and non-stationary features when $\dimension\geq2$. As illustrated in Example~\ref{ex:narrow-misplaced-set}, a relaxed feasible set can become narrow in the current context direction while remaining far from the current hidden feature. Maintaining multiple scales ensures that a larger-scale set can still provide a valid certificate of localization. This particular difficulty does not arise in either the stationary contextual setting of \citet{CLP-20} or the non-stationary identical-context setting of \Cref{sec:identical-context}, where a single price-margin scale suffices for the respective regret guarantees.

\xhdr{Shallow pricing with LEP.} Our contextual algorithms build on the geometric approach of
\citet{CLP-20}: each purchase or rejection provides a halfspace
constraint on the hidden feature vector. For a nonempty compact
convex set $\Kset\subseteq\Rspace^\dimension$, define its width
in context direction $\uvec$ as
\begin{align*}
    \width(\Kset,\uvec)
    \triangleq
    \max\nolimits_{\xvec\in\Kset}\ip{\xvec}{\uvec}
    -
    \min\nolimits_{\xvec\in\Kset}\ip{\xvec}{\uvec}.
\end{align*}
At scale $\spscale$, we maintain a relaxed feasible set
$\Sset_t^{(\spscale)}$ and an enclosing ellipsoid
$\Eset_t^{(\spscale)}$, where
\begin{align*}
    \Eset_t^{(\spscale)}
    \triangleq
    \left\{
        \xvec\in\Rspace^\dimension:
        (\xvec-\ellcenter_t^{(\spscale)})^\top
        (\ellmat_t^{(\spscale)})^{-1}
        (\xvec-\ellcenter_t^{(\spscale)})
        \leq 1
    \right\},
    \;\;
    \mbox{and}\;\;
    \width(\Eset_t^{(\spscale)},\uvec_t)
    =2\sqrt{\uvec_t^\top\ellmat_t^{(\spscale)}\uvec_t}.
\end{align*}
Here, $\ellcenter_t^{(\spscale)}$ is the center and
$\ellmat_t^{(\spscale)}$ is symmetric positive definite.
The feasible set stores the feedback constraints and determines
the exploitation price; the ellipsoid determines whether
further exploration is needed. We initialize the feasible set and its enclosing ellipsoid as
$\Sset_0=\{\xvec\in\Rspace^\dimension:
\xvec\geq0,\ \norm{\xvec}\leq1\}$
and
$\Eset_0=\{\xvec\in\Rspace^\dimension:
\norm{\xvec}\leq1\}$, respectively.

The \shallowpricing{} procedure explores at scale $\spscale$
only when
$\width(\Eset_t^{(\spscale)},\uvec_t)>4\dimension\spscale$.
It proposes the midpoint price
$\midprice_t=\ip{\ellcenter_t^{(\spscale)}}{\uvec_t}$;
our algorithm posts
$\price_t=\min\{1,\max\{0,\midprice_t\}\}$.
After observing feedback $\fb_t=\indicator{\text{Buyer $t$ accept price $\price_t$}}$, we form the relaxed halfspace
\begin{align*}
    \Hset_t^{(\spscale)}
    \triangleq
    \begin{cases}
        \{\xvec\in\Rspace^\dimension:
          \ip{\xvec}{\uvec_t}\leq\price_t+\spscale\},
        & \fb_t=0,\\
        \{\xvec\in\Rspace^\dimension:
          \ip{\xvec}{\uvec_t}\geq\price_t-\spscale\},
        & \fb_t=1.
    \end{cases}
\end{align*}
Every observation updates the relaxed feasible sets using
the actual price and feedback:
$\Sset_{t+1}^{(\spscale)}
=\Sset_t^{(\spscale)}\cap\Hset_t^{(\spscale)}$.
On an exploration round, only the selected ellipsoid is
updated, using the inferred midpoint feedback
$\virtualfb_t=\fb_t\indicator{\midprice_t\leq1}$.
The operator
$\shallowupdate(\Eset,\uvec;\spscale,\virtualfb)$
denotes the standard shallow update at the unclipped
midpoint, with acceptance if $\virtualfb=1$ and rejection
if $\virtualfb=0$.
For $\midprice_t<0$, the posted price is necessarily accepted,
and its relaxed constraint implies the midpoint-based
constraint.
For $\midprice_t>1$, the midpoint would necessarily be
rejected, and every $\xvec\in\Sset_0$ satisfies
$\ip{\xvec}{\uvec_t}\leq1<\midprice_t+\spscale$.
Thus, clipping preserves the ellipsoidal outer approximation
and the exploration guarantee.
The update formulas are given in
\Cref{app:technical-background}.

For a nonempty compact feasible set $\Sset$, its
\emph{lower-envelope price} in context direction $\uvec$ is
\begin{align*}
    \LEP(\Sset,\uvec)
    \triangleq
    \max\left\{
        0,\min\nolimits_{\xvec\in\Sset}\ip{\xvec}{\uvec}
    \right\}.
\end{align*}
This generalizes lower-endpoint pricing in the identical-context
setting. Since
$\Sset_t^{(\spscale)}\subseteq\Eset_t^{(\spscale)}$,
the ellipsoid also certifies that
$\width(\Sset_t^{(\spscale)},\uvec_t)
\leq\width(\Eset_t^{(\spscale)},\uvec_t)$.

\xhdr{Identical-context approach does not extend directly.}
A natural extension of the identical-context algorithm fixes
a single price-margin parameter $\del$ and maintains a relaxed
feasible set $\Sset_t^{(\del)}$ together with an enclosing
ellipsoid $\Eset_t^{(\del)}$.
In each round, the algorithm explores using \shallowpricing{}
if $\width(\Eset_t^{(\del)},\uvec_t)>4\dimension\del$;
otherwise, it posts
$\price_t=\LEP(\Sset_t^{(\del)},\uvec_t)$.
Every observation updates the relaxed feasible set, while
only exploration rounds update the ellipsoid.
The algorithm restarts when the window of $\Wsize$ rounds
expires, the relaxed feasible set becomes empty, or an
exploitation price is rejected.

A crucial difficulty arises from using a single price-margin
parameter.
In the identical-context setting, each accepted exploitation
round in batch $\batch_i$ incurs regret of at most
$2\del+\TVpB_i$, where $\TVpB_i$ is the within-batch value range.
With heterogeneous contexts, the analogous quantity is the
within-batch parameter diameter,
$\TVpB_i\triangleq
\max_{s,t\in\batch_i}\norm{\hparam_s-\hparam_t}$.
If $\TVpB_i\leq\del$, every parameter in the batch satisfies
all $\del$-relaxed feedback constraints.
Thus, whenever the ellipsoid is sufficiently narrow for
exploitation, the regret is at most $4\dimension\del$.
However, when $\TVpB_i>\del$, even by only a constant factor,
small parameter changes can cause a large localization error,
as the following example illustrates.
\Cref{app:single-scale-contextual} provides a formal analysis of this single-scale
extension.
Optimizing the parameters in the resulting upper bound yields
a dynamic regret guarantee of
$O(
        (\dimension^2
        +
        \dimension^{3/2}\TVUB^{1/3}\TT^{2/3})\log \TT)$.
For fixed dimension, a matching lower bound up to logarithmic
factors shows that this suboptimal rate is unavoidable under
the selected tuning.

\begin{example}[A narrow but misplaced relaxed feasible set]
\label{ex:narrow-misplaced-set}
Let $\dimension=2$, fix $\del\in(0,1/16]$, and write
$\ctxnorm=\sqrt{1+\del^2}$.
Initialize the relaxed feasible set to
$\Sset_0=\{\xvec\in\Rspace^\dimension:
\xvec\geq0,\ \norm{\xvec}\leq1\}$.
Consider two price queries within a batch.\footnote{The first
two prices are prescribed to illustrate the geometry of the
relaxed feedback constraints. We do not claim that they are
generated by the shallow-pricing midpoint rule.}

In round $1$, let
$\uvec_1=(0,1)$,
$\hparam_1=(1/2,1/4+\del)$, and
$\price_1=1/4+\del$.
The buyer accepts because $\val_1=\price_1$, giving the
relaxed constraint $\xvec_2\geq1/4$.
In round $2$, let
$\uvec_2=(\del,1)/\ctxnorm$,
$\hparam_2=(1/2,1/4-2\del)$, and
$\price_2=(1/4+\del^2)/\ctxnorm-\del$.
The value is $\val_2=(1/4-3\del/2)/\ctxnorm$, so
$\price_2-\val_2
=\del(\del+3/2-\ctxnorm)/\ctxnorm>0$.
The buyer therefore rejects, giving the relaxed constraint
$\del\xvec_1+\xvec_2\leq1/4+\del^2$.
After these observations, the relaxed feasible set is
$\Sset^{(\del)}
=\Sset_0\cap
\{\xvec\in\Rspace^\dimension:
\xvec_2\geq1/4,\ \del\xvec_1+\xvec_2\leq1/4+\del^2\}$.

From round $3$ onward, let
$\uvec_t=(1,0)$ and
$\hparam_t=(1/2,1/4+\del)$.
All contexts, parameters, and prescribed prices satisfy
the model assumptions.
The parameter diameter is $3\del$, and the total parameter
variation is $6\del$.
Yet the two relaxed constraints imply
$0\leq\xvec_1\leq\del$ for every
$\xvec\in\Sset^{(\del)}$.
Since both $(0,1/4)$ and $(\del,1/4)$ belong to this set,
$\width(\Sset^{(\del)},\uvec_t)=\del$.

For every round $t\geq3$, the true value is $1/2$, whereas
$\LEP(\Sset^{(\del)},\uvec_t)=0$.
This price is accepted and incurs regret $1/2$.
Its relaxed feedback constraint, $\xvec_1\geq-\del$,
leaves the feasible set unchanged.
Thus, if LEP continues to be used without a restart,
every subsequent round incurs regret $1/2$, despite
the parameter diameter being only $3\del$.
\end{example}

The example shows that a small parameter diameter and a
small directional width of the relaxed feasible set do not,
by themselves, imply a regret bound of
$O(\TVpB_i+\del)$ on accepted exploitation rounds.
Here, $\TVpB_i+\del=4\del$, while the regret remains $1/2$
as $\del$ tends to zero.
The two initial contexts are nearly parallel.
Combining their feedback constraints requires dividing by
$\del$, which can amplify an $O(\del)$ change in one
coordinate into a constant localization error in another.
With identical contexts, all constraints concern the same
scalar value, so this amplification does not occur.

This example illustrates a limitation of reasoning from
the relaxed feasible set alone.
It does not establish that the single-scale algorithm
generates this history: the prescribed queries need not
be its midpoint prices, and a narrow feasible set need
not have a narrow enclosing ellipsoid.
Accordingly, the example explains the obstacle to extending
the one-dimensional exploitation argument, without proving
a regret lower bound for every single-scale tuning.

This example motivates maintaining several relaxation scales,
as we explain after introducing the multiscale algorithm.

\xhdr{Algorithm for general environment: multi-scale shallow pricing.}
Our algorithm retains the two restart mechanisms from the
identical-context setting: periodic restarts controlled by
a window-size parameter $\Wsize$ and inconsistency restarts
triggered by rejected exploitation prices.
We again call the consecutive rounds between restarts a
\emph{batch}.
Within each batch, the algorithm combines shallow-pricing
exploration with lower-envelope pricing (LEP), using multiple
relaxation scales to account for the unknown amount of
parameter variation.

The algorithm is parameterized by a \emph{base price-margin
parameter $\del$}, which determines the smallest relaxation
scale.
We maintain the dyadic scale set
$\spscaleSet=\{\del,2\del,\ldots,2^{\scalelevels}\del\}$,
where $\scalelevels$ is the smallest nonnegative integer
such that $2^{\scalelevels}\del\geq\sqrt{2}$, the diameter
of the feature domain for $\dimension\geq2$.
Each scale $\spscale\in\spscaleSet$ maintains a relaxed
feasible set $\Sset_t^{(\spscale)}$ and an enclosing
ellipsoid $\Eset_t^{(\spscale)}$.
At the beginning of each batch, all feasible sets and
ellipsoids are reset to their initial states.
Each scale represents a different allowance for parameter
variation within the batch.
A scale is \emph{active} if its relaxed feasible set is
nonempty; the largest scale always remains active. We initialize
$\Sset_0=\{\xvec\in\Rspace^\dimension:
\xvec\geq0,\ \norm{\xvec}\leq1\}$
and
$\Eset_0=\{\xvec\in\Rspace^\dimension:\norm{\xvec}\leq1\}$.

For exploration, the algorithm checks the ellipsoid width
in the current context direction at every active scale.
If any active scale $\spscale$ satisfies
$\width(\Eset_t^{(\spscale)},\uvec_t)>4\dimension\spscale$,
the algorithm selects the smallest such scale and posts
the corresponding shallow-pricing price.
The resulting feedback updates the relaxed feasible sets
at all scales, using their respective margins, while only
the ellipsoid at the selected scale receives a shallow-cut
update.
Since a new context may reveal uncertainty in a different
direction, exploration and exploitation can be interleaved
within a batch.

For exploitation, the algorithm requires every active
scale to satisfy its width condition.
It then posts the lower-envelope price computed from
the smallest active scale.
This choice uses the nesting of the relaxed feasible sets:
a smaller scale gives a smaller feasible set and therefore
a weakly higher lower-envelope price.
Meanwhile, a sufficiently large scale still contains the
current hidden feature vector and certifies that the remaining
uncertainty is small.
Together, these properties control the regret on accepted
exploitation rounds even when the smallest active set
does not contain the current parameter.

As in the identical-context setting, the algorithm restarts
after at most $\Wsize$ rounds or immediately after an
exploitation price is rejected.
Periodic restarts limit the regret from changes that remain
undetected because buyers continue accepting the posted
prices.
An inconsistency restart responds to overpricing and
certifies more than $\del$ parameter variation within
the batch, allowing its frequency to be controlled by
$\TVUB$.
The window-size parameter $\Wsize$ balances the cost of
retaining outdated information against the exploration
cost of frequent restarts, while the base margin $\del$
balances pricing accuracy against exploration and
inconsistency-restart costs.
The formal algorithm is given in
\Cref{alg:multiscale-contextual}.

\begin{algorithm}[t]
\footnotesize
\caption{Multiscale shallow pricing with LEP}
\label{alg:multiscale-contextual}
\begin{algorithmic}[1]
\STATE \textbf{Input:}
$\TT\in\Nspace$, $\dimension\in\Nspace$ with
$\dimension\geq2$, $\TVUB\in[0,\TT]$
\STATE \textbf{Parameters:}
$\del\gets\max\{1/\TT,\sqrt{\dimension\TVUB/\TT}\}$,
$\Wsize\gets
\min\{\TT,\lceil\sqrt{\dimension\TT/\TVUB}\rceil\}$
\hfill$\triangleright$ set $\Wsize \gets T$ when $\TVUB = 0$
\STATE $\spscaleSet\gets
\{\del,2\del,\ldots,2^{\scalelevels}\del\}$,
where $\scalelevels$ is the smallest nonnegative integer
with $2^{\scalelevels}\del\geq\sqrt{2}$
\STATE $t\gets1$, $i\gets0$
\WHILE{$t\leq\TT$}
    \STATE $i\gets i+1$,
    $\batchend\gets\min\{\TT,t+\Wsize-1\}$
    \STATE For all $\spscale\in\spscaleSet$, set
    $\Sset_t^{(\spscale)}\gets\Sset_0$ and
    $\Eset_t^{(\spscale)}\gets\Eset_0$
    \WHILE{$t\leq\batchend$}
        \STATE Observe $\uvec_t$; set
        $\spscaleSet_t\gets
        \{\spscale\in\spscaleSet:
        \Sset_t^{(\spscale)}\neq\varnothing\}$
        \STATE $\explorescales_t\gets
        \{\spscale\in\spscaleSet_t:
        \width(\Eset_t^{(\spscale)},\uvec_t)
        >4\dimension\spscale\}$
        \IF{$\explorescales_t\neq\varnothing$}
            \STATE $\spscale_t\gets\min\explorescales_t$,
            $\midprice_t\gets
            \ip{\ellcenter_t^{(\spscale_t)}}{\uvec_t}$,
            $\price_t\gets\min\{1,\max\{0,\midprice_t\}\}$
        \ELSE
            \STATE $\spscale_t^{\min}\gets\min\spscaleSet_t$,
            $\price_t\gets
            \LEP(\Sset_t^{(\spscale_t^{\min})},\uvec_t)$
        \ENDIF
        \STATE Post $\price_t$ and observe $\fb_t$
        \STATE For all $\spscale\in\spscaleSet$, set
        $\Sset_{t+1}^{(\spscale)}
        \gets\Sset_t^{(\spscale)}\cap\Hset_t^{(\spscale)}$
        and $\Eset_{t+1}^{(\spscale)}
        \gets\Eset_t^{(\spscale)}$
        \IF{$\explorescales_t\neq\varnothing$}
            \STATE $\virtualfb_t\gets
            \fb_t\indicator{\midprice_t\leq1}$
            \STATE $\Eset_{t+1}^{(\spscale_t)}
            \gets
            \shallowupdate(
                \Eset_t^{(\spscale_t)},\uvec_t;
                \spscale_t,\virtualfb_t)$
        \ENDIF
        \STATE $t\gets t+1$
        \STATE \textbf{if}
        $\explorescales_{t-1}=\varnothing$ and
        $\fb_{t-1}=0$, \textbf{break}
        \hfill$\triangleright$ inconsistency restart
    \ENDWHILE
\ENDWHILE
\end{algorithmic}
\end{algorithm}

Before stating the dynamic regret guarantee, we revisit
Example~\ref{ex:narrow-misplaced-set} to explain how multiple scales
address the illustrated difficulty.

Insert the first two observations from
Example~\ref{ex:narrow-misplaced-set} into the relaxed feasible sets
at every scale.
At scale $\del$, the resulting set is narrow in direction
$\uvec_t=(1,0)$ but excludes the current parameter
$\hparam_t=(1/2,1/4+\del)$.
At scale $4\del$, the current parameter remains feasible
because the within-batch parameter diameter is only $3\del$.
This larger-scale set also contains $(0,1/4)$, so its width
in direction $\uvec_t$ is at least $1/2$.
For $\del<1/64$, this exceeds the exploration threshold
$4\dimension(4\del)=32\del$.
Any enclosing ellipsoid therefore also violates the width
condition, and the multiscale rule would explore after
this feedback history.
The larger scale provides a reliable certificate of
localization, while the smallest active scale determines
the exploitation price.
To see why this combination suffices, consider a batch
$\batch_i$ with parameter diameter
$\TVpB_i\triangleq
\max_{s,t\in\batch_i}\norm{\hparam_s-\hparam_t}$.
Let $\spscale_i$ be the smallest maintained scale satisfying
$\spscale_i\geq\TVpB_i$.
Then $\spscale_i\leq2\max\{\TVpB_i,\del\}$, and
$\hparam_t\in\Sset_t^{(\spscale_i)}$ throughout the batch.
Thus, scale $\spscale_i$ remains active.
Whenever the algorithm exploits, every active scale,
including $\spscale_i$, satisfies its ellipsoid width
condition.
Let $\spscale_t^{\min}$ denote the smallest active scale.
Since $\spscale_t^{\min}\leq\spscale_i$, nesting gives
$\Sset_t^{(\spscale_t^{\min})}
\subseteq\Sset_t^{(\spscale_i)}$.
Consequently, the posted price satisfies
$\price_t
=\LEP(\Sset_t^{(\spscale_t^{\min})},\uvec_t)
\geq\LEP(\Sset_t^{(\spscale_i)},\uvec_t)$.
On each accepted exploitation round,
$\val_t-\price_t
\leq\width(\Sset_t^{(\spscale_i)},\uvec_t)
\leq\width(\Eset_t^{(\spscale_i)},\uvec_t)
\leq4\dimension\spscale_i
\leq8\dimension\max\{\TVpB_i,\del\}$.
Multiple scales therefore recover a regret bound controlled
by the parameter diameter and the base margin, even when
the smallest active set excludes the current parameter.
A misplaced smaller-scale set may also cause overpricing.
The resulting rejection immediately triggers an
inconsistency restart, and these rejected rounds are
accounted for separately in the regret analysis.
The bound above applies to accepted exploitation rounds
without requiring the smallest active set to contain
the current parameter.

\xhdr{Regret guarantee and its analysis.}
We now state the dynamic regret guarantee of
\Cref{alg:multiscale-contextual}.

\begin{theorem}
\label{thm:multiscale-contextual}
In the contextual environment, given any non-stationarity upper bound
$\TVUB\in[0,\TT]$, set
$\Wsize=\min\{\TT,\lceil\sqrt{\dimension\TT/\TVUB}\rceil\}$
and
$\del=\max\{1/\TT,\sqrt{\dimension\TVUB/\TT}\}$ with the convention $\Wsize=\TT$ when $\TVUB=0$.
Then the dynamic regret of \Cref{alg:multiscale-contextual}
is at most
$O((\dimension^2+\dimension^{3/2}\sqrt{\TT\TVUB})
\log^2\TT)$.

Moreover, for every $\TVUB\in[0,\TT]$ and every deterministic
online algorithm, there exists an admissible instance with
total variation at most $\TVUB$ on which the algorithm
incurs dynamic regret of at least
$\Omega(\sqrt{\TT\TVUB}+\log\log\TT)$.
\end{theorem}
Combining the upper and lower bounds in
\Cref{thm:multiscale-contextual}, the regret guarantee of
\Cref{alg:multiscale-contextual} is optimal among deterministic
online algorithms, up to logarithmic factors in $\TT$, for
every $\TVUB\in[0,\TT]$ when $\dimension$ is fixed.
The term $O(\dimension^2\log^2\TT)$ accounts for the initial
exploration cost and ensures that the guarantee also covers
the stationary case $\TVUB=0$.
The remaining term scales as $\sqrt{\TT\TVUB}$, matching
the variation-dependent lower bound up to logarithmic
and dimension-dependent factors.
We do not claim that the dependence on $\dimension$ is
optimal.

The lower bound follows directly from
\Cref{thm:1d-search}, since the identical-context environment
is a special case of the contextual environment.
Specifically, any scalar value sequence
$\{\val_t\}_{t\in[\TT]}$ can be embedded by setting
$\uvec_t=(1,0,\ldots,0)$ and
$\hparam_t=(\val_t,0,\ldots,0)$.
For the upper bound, it suffices to analyze
$\TVUB\leq\TT/\dimension$.
When $\TVUB>\TT/\dimension$, the stated bound follows
from the trivial regret bound $\TT$, since
$\dimension^{3/2}\sqrt{\TT\TVUB}\geq\TT$.
We next sketch the analysis for the nontrivial regime;
the full proof is deferred to the \Cref{app:multiscale-contextual}.

Consider the regime $\TVUB\leq\TT/\dimension$.
Let $\batch_1,\ldots,\batch_{\batchcount}$ denote the batches,
and define their parameter diameters by
$\TVpB_i\triangleq
\max_{s,t\in\batch_i}\norm{\hparam_s-\hparam_t}$.
Since the batches are disjoint intervals,
$\sum_{i=1}^{\batchcount}\TVpB_i\leq\TVUB$.
We bound the regret from exploration and exploitation
separately.

We first bound the number of batches.
Periodic restarts contribute at most $\TT/\Wsize$
completed batches.
Every rejected exploitation round certifies more than
$\del$ parameter variation within its batch: otherwise,
the current parameter would remain feasible at scale
$\del$, and the smallest-scale LEP could not exceed
the current value.
Thus, there are at most $\TVUB/\del$ inconsistency
restarts, giving
$\batchcount\leq1+\TT/\Wsize+\TVUB/\del
=O(1+\sqrt{\TT\TVUB/\dimension})$.

For exploration, the shallow-pricing exploration bound
gives at most $O(\dimension^2\log\TT)$ exploration rounds
per scale in each batch.
Since $\del\geq1/\TT$, the number of maintained scales
is $O(\log\TT)$.
Each exploration round incurs at most one unit of regret,
so the total exploration regret is at most
$O(\batchcount\dimension^2\log^2\TT)
=O((\dimension^2+\dimension^{3/2}\sqrt{\TT\TVUB})
\log^2\TT)$.

For exploitation, let $\spscale_i$ be the smallest maintained
scale satisfying $\spscale_i\geq\TVpB_i$.
As established above, this scale remains active throughout
batch $\batch_i$, and each accepted exploitation round
incurs regret of at most
$4\dimension\spscale_i
\leq8\dimension\max\{\TVpB_i,\del\}$.
Using $|\batch_i|\leq\Wsize$ and
$\sum_{i=1}^{\batchcount}\TVpB_i\leq\TVUB$,
the total regret from accepted exploitation rounds is at most
$8\dimension\TT\del
+8\dimension\sum_{i=1}^{\batchcount}
|\batch_i|\TVpB_i
\leq8\dimension\TT\del+8\dimension\Wsize\TVUB
=O(\dimension+\dimension^{3/2}\sqrt{\TT\TVUB})$.
Each rejected exploitation round incurs at most one unit
of regret, and there are at most $\TVUB/\del$ such rounds.
Their total regret is therefore at most
$\TVUB/\del\leq\sqrt{\TT\TVUB/\dimension}$.

Combining these bounds yields
$\Reg[\ALG]
=O((\dimension^2+\dimension^{3/2}\sqrt{\TT\TVUB})
\log^2\TT)$,
as claimed.

\section*{AI Use Statement}

Generative AI tools were used as an assistive tool during the development of
this work, including for improving exposition and organization, checking
notation and internal consistency, discussing proof presentation and examples,
and assisting with literature exploration. All research ideas, mathematical
claims, proofs, and cited references included in the final manuscript were
reviewed and verified by the authors. The authors take full responsibility for
the final content of the paper.

\bibliographystyle{apalike}
\bibliography{refs}

\clearpage
\appendix
\crefalias{section}{appendix}

\section{Technical Background: Geometric Search and Shallow Pricing}
\label{app:technical-background}

\subsection{Background: geometric search in feature-based pricing}
\label{subsec:featdp-background}

Our contextual algorithms build on the geometric viewpoint introduced in
feature-based dynamic pricing by \cite{CLP-20}. In that model, each arriving
product is described by an observed feature vector, the buyer's valuation is
linear in an unknown parameter, and the seller observes only binary
accept/reject feedback after posting a price. Thus, each observation induces a
halfspace constraint on the unknown parameter. Pricing can therefore be viewed
as a geometric search problem: the seller maintains a region of parameters
that remain consistent with past feedback, and each new context direction asks
for the range of possible valuations induced by this region.

For a convex set $K\subseteq\reals^d$ and a context direction $\uvec$, define
the directional width of $K$ along $\uvec$ as
\[
    \width(K,\uvec)
    :=
    \max_{x\in K}\ip{x}{\uvec}
    -
    \min_{x\in K}\ip{x}{\uvec}.
\]
This is the length of the interval of valuations that are still possible in
context $\uvec$ when the hidden parameter is only known to lie in $K$.
Equivalently, if
\[
    b^-(K,\uvec)
    :=
    \min_{x\in K}\ip{x}{\uvec},
    \qquad
    b^+(K,\uvec)
    :=
    \max_{x\in K}\ip{x}{\uvec},
\]
then
\[
    \width(K,\uvec)
    =
    b^+(K,\uvec)-b^-(K,\uvec).
\]

The algorithmic primitive in \cite{CLP-20} is \textsc{EllipsoidPricing}.
Instead of working directly with a possibly complicated polyhedral uncertainty
set, the algorithm maintains an ellipsoidal outer approximation
\[
    \mathcal E_\round
    =
    \mathcal E(A_\round,c_\round)
    :=
    \left\{
        x\in\reals^d:
        (x-c_\round)^\top A_\round^{-1}(x-c_\round)\le1
    \right\},
\]
where $c_\round$ is the center and $A_\round$ is a positive definite matrix. When this primitive is run at scale $\lambda$ in our algorithms, we
write these quantities as $E_t^{(\lambda)}$,
$c_t^{(\lambda)}$, and $A_t^{(\lambda)}$.
For a context $\uvec_\round$, the minimum and maximum possible valuations over
this ellipsoid have the closed forms
\[
    b_\round^-
    =
    \ip{c_\round}{\uvec_\round}
    -
    \sqrt{\uvec_\round^\top A_\round\uvec_\round},
    \qquad
    b_\round^+
    =
    \ip{c_\round}{\uvec_\round}
    +
    \sqrt{\uvec_\round^\top A_\round\uvec_\round}.
\]
Thus
\[
    \width(\mathcal E_\round,\uvec_\round)
    =
    b_\round^+-b_\round^-
    =
    2\sqrt{\uvec_\round^\top A_\round\uvec_\round}.
\]

Given a target accuracy $\eta>0$, \textsc{EllipsoidPricing} behaves as
follows. If
\[
    b_\round^+-b_\round^-\le\eta,
\]
the algorithm exploits by posting the lower endpoint
\[
    \price_\round=b_\round^-.
\]
If instead
\[
    b_\round^+-b_\round^->\eta,
\]
the algorithm explores by posting the midpoint price
\[
    \price_t
    =
    \frac{b_t^-+b_t^+}{2}
    =
    \ip{c_t}{\uvec_t}.
\]

An acceptance, $\fb_\round=1$, implies the central halfspace constraint
\[
    \ip{x}{\uvec_\round}\ge\price_\round,
\]
while a rejection, $\fb_\round=0$, implies
\[
    \ip{x}{\uvec_\round}\le\price_\round.
\]
The next ellipsoid is chosen as the minimum-volume ellipsoid containing the
corresponding half-ellipsoid.

More explicitly, define
\[
    q_\round
    :=
    \frac{A_\round\uvec_\round}
    {\sqrt{\uvec_\round^\top A_\round\uvec_\round}}.
\]
For an acceptance, the central-cut update is
\[
    c_{\round+1}
    =
    c_\round+\frac{1}{d+1}q_\round,
\]
and for a rejection it is
\[
    c_{\round+1}
    =
    c_\round-\frac{1}{d+1}q_\round.
\]
In both cases,
\[
    A_{\round+1}
    =
    \frac{d^2}{d^2-1}
    \left(
        A_\round
        -
        \frac{2}{d+1}
        q_\round q_\round^\top
    \right).
\]
The key fact behind the exploration bound is that each central cut decreases
the ellipsoid volume by a constant factor, while exploration occurs only when
the current direction has large ellipsoidal width.

The robust primitive we use is \textsc{ShallowPricing}, also from
\cite{CLP-20}. It modifies the central-cut update by adding a safety buffer to
the inferred halfspace. 
At scale $\lambda>0$, let
\[
    \widetilde{\price}_t
    :=
    \frac{b_t^-+b_t^+}{2}
    =
    \ip{c_t}{\uvec_t}
\]
denote the standard, unclipped midpoint proposal. Under the standard
\textsc{ShallowPricing} update, an acceptance produces the shifted constraint
\[
    \ip{x}{\uvec_\round}
    \ge
    \widetilde{\price}_t-\lambda,
\]
and a rejection produces the shifted constraint
\[
    \ip{x}{\uvec_\round}
    \le
    \widetilde{\price}_t+\lambda.
\]
Thus the update no longer cuts the ellipsoid through its center; it is a
shallow cut.

The shallow-cut update can also be written in closed form. Let
\[
    \alpha_\round
    :=
    -\frac{\lambda}
    {\sqrt{\uvec_\round^\top A_\round\uvec_\round}}.
\]
The quantity $\alpha_\round$ is the signed depth of the cut. When the
algorithm explores only if
\[
    \width(\mathcal E_\round,\uvec_\round)>4d\lambda,
\]
we have
\[
    \sqrt{\uvec_\round^\top A_\round\uvec_\round}
    >
    2d\lambda,
\]
and hence
\[
    \alpha_\round
    >
    -\frac{1}{2d}.
\]
This is the regime in which the shallow-cut ellipsoid update is valid.

For an acceptance, the shallow-cut center update is
\[
    c_{\round+1}
    =
    c_\round
    +
    \frac{1+d\alpha_\round}{d+1}
    q_\round,
\]
whereas for a rejection,
\[
    c_{\round+1}
    =
    c_\round
    -
    \frac{1+d\alpha_\round}{d+1}
    q_\round.
\]
In both cases, the shape matrix is updated as
\[
    A_{\round+1}
    =
    \frac{d^2}{d^2-1}
    (1-\alpha_\round^2)
    \left(
        A_\round
        -
        \frac{2(1+d\alpha_\round)}
        {(d+1)(1+\alpha_\round)}
        q_\round q_\round^\top
    \right).
\]
\paragraph{Clipping the exploration price to $[0,1]$.}
Our algorithms post
\[
    \price_t := \min\{1,\max\{0,\widetilde\price_t\}\}.
\]
Let
\[
    \widetilde f_t
    :=
    f_t\,\mathbf{1}\{\widetilde\price_t\le 1\}
\]
denote the inferred feedback at the unclipped midpoint. We denote the
corresponding ellipsoid update by
\[
    E_{t+1}
    =
    \operatorname{ShallowUpdate}
    (E_t,u_t;\lambda,\widetilde f_t).
\]
The relaxed feasible
set is updated using the actual posted price $\price_t$ and observed feedback
$f_t$, while the ellipsoid is updated using the standard shallow cut
associated with $\widetilde\price_t$ and $\widetilde f_t$.

If $0\le \widetilde\price_t\le1$, then
$\price_t=\widetilde\price_t$ and $\widetilde f_t=f_t$, so the standard
shallow-pricing update applies unchanged.

Suppose instead that $\widetilde\price_t<0$. Since realized valuations are
nonnegative, $\price_t=0$ is necessarily accepted, and the actual
$\lambda$-relaxed feedback constraint is
\[
    \ip{x}{\uvec_t}\ge-\lambda.
\]
Since $\widetilde\price_t<0$,
\[
    \{x:\ip{x}{\uvec_t}\ge-\lambda\}
    \subseteq
    \{x:\ip{x}{\uvec_t}\ge\widetilde\price_t-\lambda\}.
\]
Thus the actual relaxed constraint is stronger than the standard acceptance
constraint associated with $\widetilde\price_t$.

Finally, suppose that $\widetilde\price_t>1$. Then
$\widetilde f_t=0$. Since every $x\in S_0$ satisfies
$\ip{x}{\uvec_t}\le1$, we have
\[
    \ip{x}{\uvec_t}\le1<\widetilde\price_t+\lambda.
\]
Hence the updated relaxed feasible set is contained in the standard rejection
halfspace
\[
    \{x:\ip{x}{\uvec_t}\le\widetilde\price_t+\lambda\},
\]
regardless of the feedback observed at the clipped price $\price_t=1$.

Therefore, in all cases, the standard shallow cut associated with
$(\widetilde\price_t,\widetilde f_t)$ contains the updated relaxed feasible
set. Thus clipping the posted price to $[0,1]$ preserves the ellipsoidal
outer-approximation invariant and the usual shallow-cut exploration guarantee.

In our non-stationary setting, the relaxation parameter $\lambda$ has a
different interpretation from statistical noise: it is a robustness margin
that allows past feedback constraints to remain compatible with a parameter
that may have moved within the current batch. The single-scale baseline uses
one fixed radius $\lambda=\del$, while our final algorithm maintains a dyadic
hierarchy of such radii. The \textsc{ShallowPricing} primitive itself is not
our contribution; our use of it together with scale-dependent relaxed feasible
sets and LEP pricing is described in \Cref{sec:contextual}.

\subsection{Relaxed feasible sets versus ellipsoidal outer approximations}
\label{app:relaxed-vs-ellipsoid}

In our non-stationary algorithms, the relaxed feasible set and its ellipsoidal
outer approximation have distinct roles. At each relaxation scale $\lambda$
we maintain
\[
    \Sset_t^{(\lambda)}
    \subseteq
    \mathcal E_t^{(\lambda)}.
\]
The relaxed feasible set is the intersection of the shifted feedback
halfspaces and is therefore the object used for LEP pricing. The ellipsoid is
used only to decide whether the current context direction is sufficiently
uncertain to require exploration and, on such a round, to perform the
shallow-cut update. In particular,
\[
    \width(\Sset_t^{(\lambda)},\uvec_t)
    \le
    \width(\mathcal E_t^{(\lambda)},\uvec_t).
\]

On a multiscale exploration round at scale $\lambda_t$, only
$\mathcal E_t^{(\lambda_t)}$ is shallow-updated. This is because the ellipsoid update uses
$\operatorname{ShallowUpdate}$ with the unclipped midpoint
$\widetilde p_t$ and its inferred feedback $\widetilde f_t$.
Since $\widetilde p_t$ is centered with respect to
$E_t^{(\lambda_t)}$, the shallow-cut depth guarantee above applies
directly at that scale. For every other scale $\lambda$, leaving the ellipsoid
unchanged preserves the outer-approximation invariant because
\[
    \Sset_{t+1}^{(\lambda)}
    \subseteq
    \Sset_t^{(\lambda)}
    \subseteq
    \mathcal E_t^{(\lambda)}
    =
    \mathcal E_{t+1}^{(\lambda)}.
\]

Throughout the remainder of the analysis, let
\[
    R:=\sqrt{2},
\]
which upper-bounds the diameter of the initial feasible set $S_0$.

\begin{lemma}[Shallow-cut exploration bound]
\label[lemma]{lem:generic-shallow-exploration}
Fix a relaxation scale $\lambda>0$ and an exploration threshold
$\eta$ satisfying
\[
    4d\lambda\le \eta < 20R(d+1).
\]
If an ellipsoid is updated only on rounds for which
\[
    \width(E_t,u_t)>\eta,
\]
using the standard $\lambda$-shifted shallow cut associated with the
unclipped midpoint proposal, then within a fixed batch the number of
such updates is at most
\[
    20d^2
    \ln\left(
        \frac{20R(d+1)}{\eta}
    \right).
\]
\end{lemma}

\begin{proof}

Inspecting the proof of Theorem~3 of \citet{CLP-20},
the exploration-count argument uses $\eta\ge 4d\lambda$ to ensure
that the shallow-cut depth satisfies $\alpha_t\ge-1/(2d)$, and then gives
\[
    N
    \le
    20d^2
    \ln\left(
        \frac{20R(d+1)}{\eta}
    \right).
\]

The clipping argument above ensures that the ellipsoid update used here
is the same valid shallow update at the unclipped midpoint.
\end{proof}

\section{Identical Context Upper Bound}
\label{app:noncontextual-upper}
\subsection{Complete proof of \Cref{thm:1d-search}}
\label{app:proof-identical-context-upper}

\paragraph{Illustration.}
During exploration, binary search progressively narrows
$[p_i^{\min},p_i^{\max}]$ according to the observed accept/reject feedback.
The final endpoints retain accepted/rejected witness rounds that are used in
the regret analysis below. Once its width is at most $\del$, the algorithm switches
to LEP with a margin and repeatedly posts
$\punderi=\max\{\pimin-\del,0\}$. The margin makes a rejection informative:
if this LEP price is rejected, the valuation must have decreased sufficiently
relative to an earlier exploration round, and the resulting reset can be
charged to the total variation budget. The hard window-size reset, in
contrast, limits the duration of undetectable upward movements. These are
precisely the two ingredients used in the regret analysis below.

\begin{proof}
Let
\[
    \batch_1=\{1,\ldots,t_1\},
    \quad
    \batch_2=\{t_1+1,\ldots,t_2\},
    \quad
    \ldots,
    \quad
    \batch_r=\{t_{r-1}+1,\ldots,\TT\}
\]
denote the batches generated by the algorithm, i.e., consecutive rounds
between resets. Let $\bReg_\phase$ denote the regret incurred in batch
$\batch_\phase$, and let $E_\phase\subseteq\batch_\phase$ denote its
exploration rounds. By construction,
\[
    |\batch_\phase|\le\Wsize.
\]
Also define the variation diameter of batch $\batch_\phase$ by
\[
    \Delta_\phase
    :=
    \max_{j,k\in\batch_\phase}
    |\val_j-\val_k|.
\]
For every exploitation round
$j\in\batch_\phase\setminus E_\phase$, the algorithm posts the same LEP price
with margin,
\[
    \price_j=\punderi=\max\{\pimin-\del,0\}.
\]

For any batch that reaches exploitation, let
\[
    L_i:=p_i^{\min},
    \qquad
    U_i:=p_i^{\max}
\]
denote the final binary-search interval, so that
\[
    U_i-L_i\le\delta.
\]
We first bound the regret on an accepted LEP round
$j\in B_i\setminus E_i$. Recall that
\[
    p_j=\punderi=\max\{L_i-\delta,0\}.
\]

If $U_i<1$, then the final upper endpoint must have been created by a
rejected binary-search query. Hence there exists an exploration round
$t_1\in E_i$ such that
\[
    v_{t_1}<U_i.
\]
Therefore, by the definition of $\Delta_i$,
\[
    v_j
    \le
    v_{t_1}+\Delta_i
    <
    U_i+\Delta_i
    \le
    L_i+\delta+\Delta_i.
\]
Consequently,
\[
    v_j-p_i
    \le
    \Delta_i+2\delta.
\]
Indeed, if $L_i\ge\delta$, then $p_i=L_i-\delta$ and the claim follows
directly; while if $L_i<\delta$, then $p_i=0$ and
\[
    v_j
    <
    L_i+\delta+\Delta_i
    <
    \Delta_i+2\delta.
\]

If instead $U_i=1$, then $U_i-L_i\le\delta$ implies
$L_i\ge1-\delta$. Hence
\[
    p_i=\max\{L_i-\delta,0\}\ge 1-2\delta,
\]
and since $v_j\le1$,
\[
    v_j-p_i\le2\delta\le\Delta_i+2\delta.
\]
Thus every accepted LEP round incurs regret at most
$\Delta_i+2\delta$.

The regret of batch $\batch_\phase$ can be decomposed as
\begin{align*}
    \bReg_\phase
    &=
    \sum_{j\in\batch_\phase}
    \Big(
        \val_j-\Indic{\price_j\le\val_j}\price_j
    \Big)
    \\
    &=
    \sum_{j\in E_\phase}
    \Big(
        \val_j-\Indic{\price_j\le\val_j}\price_j
    \Big)
    +
    \sum_{j\in\batch_\phase\setminus E_\phase}
    \Big(
        \val_j-\Indic{\price_j\le\val_j}\price_j
    \Big)
    \\
    &\le
    |E_\phase|
    +
    |\batch_\phase|(\Delta_\phase+2\del) +1
    \\
    &\le
    O(\log(\TT)
    +
    |\batch_\phase|\del
    +
    \Wsize\Delta_\phase).
\end{align*}
Here we used that binary search requires at most $\log(\TT)$ exploration
rounds, and that on every exploitation round before a reset the regret is at
most $\Delta_\phase+2\del$.

Summing over all batches gives
\[
\begin{aligned}
    \cRT
    =
    \sum_{\phase=1}^r\bReg_\phase
    &\le
    \sum_{\phase=1}^r
    O\left(
        |\batch_\phase|\del
        +
        \Wsize\Delta_\phase
        +
        \log(\TT)
    \right)
    \\
    &=
    O(\TT\del
    +
    \Wsize\sum_{\phase=1}^r\Delta_\phase
    +
    r\log(\TT)).
\end{aligned}
\]
Since the batches are disjoint intervals,
\[
    \sum_{\phase=1}^r\Delta_\phase
    \le
    \TV(\mathbf v)
    \le
    \TVUB.
\]
It therefore remains to bound the number of batches.

\begin{lemma}
\label[lemma]{lem:upperboundbatchnumbers}
The two reset mechanisms imply
\[
    r
    \le
    \frac{\TVUB}{\del}
    +
    \frac{\TT}{\Wsize}
    +
    1.
\]
\end{lemma}

\begin{proof}
Every batch $\batch_\phase$ ends for one of the following reasons:
\begin{itemize}
    \item the window-size cap is reached, i.e.,
    $|\batch_\phase|=\Wsize$;
    \item $\phase=r$, so this is the final incomplete batch;
    \item an inconsistency occurs on an LEP round and triggers an immediate
    reset.
\end{itemize}

There are at most $\TT/\Wsize$ batches of the first type and at most one
batch of the second type. We now show that there are at most
$\TVUB/\del$ inconsistency resets.

Consider a batch $\batch_\phase$ that ends because an LEP price is rejected
at round $t_\phase$. Since the algorithm resets only after the binary-search
phase, this means that LEP has become overpricing:
\[
    \val_{t_\phase}
    <
    \price_{t_\phase}
    =
    \punderi.
\]
Moreover, an inconsistency can only occur when $\punderi>0$, so
\[
    \punderi
    =
    \pimin-\del.
\]
Hence
\[
    \val_{t_\phase}
    <
    \pimin-\del.
\]

During binary search, the lower endpoint $\pimin$ can increase only after an
accepted query. Therefore, there exists an exploration round
$j\in E_\phase$ at which a price equal to the resulting lower endpoint was
accepted, so
\[
    \pimin
    \le
    \val_j.
\]
Consequently,
\[
    \val_{t_\phase}+\del
    <
    \pimin
    \le
    \val_j,
\]
and therefore
\[
    \del
    <
    |\val_j-\val_{t_\phase}|
    \le
    \Delta_\phase.
\]
In particular,
\[
    \del
    <
    \sum_{k=t_{\phase-1}+1}^{t_\phase-1}
    |\val_{k+1}-\val_k|.
\]
Thus every inconsistency reset can be charged to more than $\del$ total
variation inside its batch. Since these batches are disjoint and
\[
    \TV(\mathbf v)\le\TVUB,
\]
there can be at most $\TVUB/\del$ such resets. Combining the three cases gives
\[
    r
    \le
    \frac{\TT}{\Wsize}
    +
    \frac{\TVUB}{\del}
    +
    1.
\]
\end{proof}

Using \Cref{lem:upperboundbatchnumbers}, we obtain
\[
\begin{aligned}
    \cRT
    &\le
    O(\TT\del
    +
    \Wsize\TVUB
    +
    \left(
        \frac{\TT}{\Wsize}
        +
        \frac{\TVUB}{\del}
        +
        1
    \right)
    \log(\TT)).
\end{aligned}
\]
For notational simplicity, in the tuning calculation below we ignore the
ceiling in the definition of $\Wsize$. This affects the bounds only by an
absolute constant factor: for $x\ge 1$,
\[
    x\le \lceil x\rceil\le 2x,
\]
and the additional truncation by $\TT$ preserves this constant-factor
relation. Thus replacing
$\min\{\TT,\lceil \sqrt{\TT/\TVUB}\rceil\}$ by
$\min\{T,\sqrt{\TT/\TVUB}\}$ does not affect the stated big-$O$ regret bound.
Setting
\[
    \Wsize
    =
    \min \{ \TT, \sqrt{\frac{\TT}{\TVUB}} \},
    \qquad
    \del
    =
    \max \{ \frac{1}{\TT},\sqrt{\frac{\TVUB}{\TT}}\},
\]
gives
\[
\begin{aligned}
    \cRT
    &\le
    O(\TT\del
    +
    \Wsize\TVUB
    +
    \left(
        \frac{\TT}{\Wsize}
        +
        \frac{\TVUB}{\del}
    \right)\log(\TT)
    +
    \log(\TT))
    \\
    &=
    O(\sqrt{\TT\TVUB}\log(\TT)
    +
    \log(\TT)).
\end{aligned}
\]

for the case $\TVUB>\frac{1}{\TT}$ and when $\TVUB\le\frac{1}{\TT}$ we get 
\[
\begin{aligned}
    \cRT
    &\le
    O(\TT\del
    +
    \Wsize\TVUB
    +
    \left(
        \frac{\TT}{\Wsize}
        +
        \frac{\TVUB}{\del}
    \right)\log(\TT)
    +
    \log(\TT))
    \\
    &=
O(1
    +
    \log(\TT))
    =
    O(\log(\TT)).
\end{aligned}
\]

This proves the claimed dynamic-regret bound.
\end{proof}

\begin{remark}
Alternatively for $\TVUB>0$ choosing
\[
    a
    =
    \min\left\{
        \TT,\sqrt{\frac{\TT\log \TT}{\TVUB}}
    \right\},
    \qquad
    \delta
    =
    \max\left\{
        \frac1\TT,\sqrt{\frac{\TVUB\log \TT}{\TT}}
    \right\},
\]
in the bound above yields the slightly sharper guarantee
\[
    \Reg[\ALG]
    =
    O\!\left(
        \sqrt{\TT\TVUB\log \TT}+\log \TT
    \right).
\]
For $\TVUB=0$, we take $\Wsize=\TT$ and $\delta=1/\TT$.
\end{remark}

\section{Single-Scale Contextual Baseline}
\label{app:single-scale-contextual}

This appendix gives the complete single-scale contextual algorithm and its
analysis. As discussed in \Cref{sec:contextual}, this is the
most direct geometric extension of the identical-context warm-up: it maintains
a single relaxed feasible set at radius $\del$, uses shallow pricing with the midpoint proposal clipped to $[0,1]$, and exploits using LEP.

Within a batch, let $\Sset_t$ denote the set of parameters consistent with the
exact feedback observed so far. In parallel, let
$\SExtendedset_t=\Sset_t^{(\del)}$ denote the $\del$-relaxed feasible set.
After posting $\price_t$ in context $\uvec_t$, define
\[
\Hset_t =
\begin{cases}
\{x:\ip{x}{\uvec_t}\le \price_t\}, & \fb_t=0,\\
\{x:\ip{x}{\uvec_t}\ge \price_t\}, & \fb_t=1,
\end{cases}
\qquad
\HExtendedset_t =
\begin{cases}
\{x:\ip{x}{\uvec_t}\le \price_t+\del\}, & \fb_t=0,\\
\{x:\ip{x}{\uvec_t}\ge \price_t-\del\}, & \fb_t=1.
\end{cases}
\]
The feasible sets are updated according to
\[
    \Sset_{t+1}
    =
    \Sset_t\cap\Hset_t,
    \qquad
    \SExtendedset_{t+1}
    =
    \SExtendedset_t\cap\HExtendedset_t.
\]

The algorithm also maintains an ellipsoidal outer approximation
\[
    \SExtendedset_t\subseteq\mathcal E_t.
\]
As in the multiscale algorithm, the relaxed feasible set determines the LEP
price, while the ellipsoid determines whether exploration is necessary.

\begin{algorithm}[t]
\footnotesize
\caption{Single-scale shallow pricing with LEP}
\label{alg:single-scale-contextual}
\begin{algorithmic}[1]

\STATE \textbf{Input:}
$\TT\in\Nspace$, $\dimension\in\Nspace$ with
$\dimension\geq2$, $\TVUB\in[0,\TT]$

\STATE \textbf{Parameters:}
$\del\gets
\max\{1/\TT,\dimension^{1/2}(\TVUB/\TT)^{1/3}\}$,
$\Wsize\gets
\min\{\TT,\lceil
\dimension^{1/2}(\TT/\TVUB)^{1/3}
\rceil\}$,
$\eta\gets4\dimension\del$
\hfill$\triangleright$ set $\Wsize \gets T$ when $\TVUB = 0$

\STATE $t\gets1$, $i\gets0$

\WHILE{$t\leq\TT$}

    \STATE $i\gets i+1$,
    $\batchend\gets\min\{\TT,t+\Wsize-1\}$

    \STATE Set
    $\Sset_t\gets\Sset_0$,
    $\SExtendedset_t\gets\Sset_0$, and
    $\Eset_t\gets\Eset_0$

    \WHILE{$t\leq\batchend$}

        \STATE \textbf{if}
        $\SExtendedset_t=\varnothing$, \textbf{break}
        \hfill$\triangleright$ inconsistency restart

        \STATE Observe $\uvec_t$

        \IF{$\width(\Eset_t,\uvec_t)>\eta$}

            \STATE
            $\midprice_t\gets
            \ip{\ellcenter_t}{\uvec_t}$,
            $\price_t\gets
            \min\{1,\max\{0,\midprice_t\}\}$

        \ELSE

            \STATE
            $\price_t\gets
            \LEP(\SExtendedset_t,\uvec_t)$

        \ENDIF

        \STATE Post $\price_t$ and observe $\fb_t$

        \STATE Construct $\Hset_t$ and $\HExtendedset_t$
        according to $\price_t$ and $\fb_t$

        \STATE
        $\Sset_{t+1}\gets\Sset_t\cap\Hset_t$,
        $\SExtendedset_{t+1}
        \gets\SExtendedset_t\cap\HExtendedset_t$,
        and $\Eset_{t+1}\gets\Eset_t$

        \IF{$\width(\Eset_t,\uvec_t)>\eta$}

            \STATE
            $\virtualfb_t\gets
            \fb_t\indicator{\midprice_t\leq1}$

            \STATE
            $\Eset_{t+1}\gets
            \shallowupdate(
                \Eset_t,\uvec_t;
                \del,\virtualfb_t)$

        \ENDIF

        \STATE $t\gets t+1$

        \STATE \textbf{if}
        $\width(\Eset_{t-1},\uvec_{t-1})\leq\eta$
        and $\fb_{t-1}=0$, \textbf{break}
        \hfill$\triangleright$ inconsistency restart

    \ENDWHILE

\ENDWHILE

\end{algorithmic}
\end{algorithm}

The empty-set restart in \Cref{alg:single-scale-contextual} is a
technical form of inconsistency reset: if the relaxed feedback constraints
become mutually infeasible, the current batch is discarded. As shown below,
such a restart, like a rejected LEP round, can be charged to at least $\del$
variation within the batch.

\begin{theorem}[Single-scale contextual regret]
\label{thm:single-scale-contextual}
For $4d\delta\le \eta<20R(d+1)$, \Cref{alg:single-scale-contextual} has dynamic regret

\[
\cRT
    \le
    O(d^2
        \log\frac{20R(d+1)}{\eta})
    \left(
        \frac{\TT}{\Wsize}
        +
        \frac{\TVUB}{\del}
        +
        1
    \right)
    +
    \TT\eta
    +
    \Wsize\frac{\TVUB}{\del}.
\]
with the convention $a=\TT$ when $\TVUB=0$.
With the parameter choices
\[
\Wsize
    =
    \min \{\TT,\lceil d^{1/2}\left(\frac{\TT}{\TVUB}\right)^{1/3}\rceil\},
    \qquad
    \del
    =
    \max \{ \frac{1}{\TT},
     d^{1/2}\left(\frac{\TVUB}{\TT}\right)^{1/3}\},
    \qquad
    \eta
    =
    4d\del,
\]
the regret is 
\[
    \cRT
    \le
    O\!\left(
        (d^2 + d^{3/2}\TVUB^{1/3}\TT^{2/3})\log (\TT)
    \right).
\]
\end{theorem}

\subsection{Shallow-pricing exploration bound}

\begin{lemma}[Shallow-pricing exploration bound]
\label[lemma]{lem:single-scale-exploration}
Suppose
$
    4d\delta\le\eta<20R(d+1).
$ Within any fixed batch, the
number of exploration rounds in \Cref{alg:single-scale-contextual}
is at most
\[
    O\!\left(
        \dimension^2
        \log\frac{20R(\dimension+1)}{\eta}
    \right).
\]
\end{lemma}

\begin{proof}
Consider an exploration round $t$. By
\Cref{alg:single-scale-contextual},
\[
    \width(\mathcal E_t,\uvec_t)>\eta.
\]
Writing
\[
    \Eset_t=E(A_t,\ellcenter_t),
\]
we have
\[
    \width(\mathcal E_t,\uvec_t)
    =
    2\sqrt{\uvec_t^\top A_t\uvec_t},
\]
and therefore
\[
    \sqrt{\uvec_t^\top A_t\uvec_t}
    >
    \frac{\eta}{2}.
\]

Let
\[
    \midprice_t:=\ip{\ellcenter_t}{\uvec_t}.
\]
The posted exploration price is
\[
    \price_t=\min\{1,\max\{0,\midprice_t\}\},
\]
and the ellipsoid is updated using
\[
    \virtualfb_t
    =
    \fb_t\indicator{\midprice_t\le1}.
\]
By \Cref{app:technical-background}, this clipping preserves the
ellipsoidal outer-approximation invariant and the validity of the
standard shallow update.

The normalized depth of this shallow cut satisfies
\[
    \frac{\del}
    {\sqrt{\uvec_t^\top A_t\uvec_t}}
    <
    \frac{2\del}{\eta}
    \le
    \frac{1}{2d}.
\]
Thus every exploration round uses a valid shallow ellipsoid update, and the
standard shallow-cut potential argument recalled in
\Cref{app:technical-background} gives the displayed bound.
\end{proof}

\subsection{Reset count}

\begin{lemma}[Reset count]
\label[lemma]{lem:single-scale-reset}
The number of non-window resets in
\Cref{alg:single-scale-contextual} is at most
$\TVUB/\del$. Consequently, the total number of batches satisfies
\[
    r
    \le
    1+\frac{\TT}{\Wsize}+\frac{\TVUB}{\del}.
\]
\end{lemma}

\begin{proof}
First consider a rejected LEP round $t$. Rejection means
\[
    \ip{\hparam_t}{\uvec_t}<\price_t.
\]
Since the realized valuation is nonnegative, this implies $\price_t>0$.
Hence the clipping at zero is inactive and
\[
    \price_t
    =
    \operatorname{LEP}(\SExtendedset_t,\uvec_t)
    =
    \min_{x\in\SExtendedset_t}\ip{x}{\uvec_t}.
\]
It follows that
\[
    \hparam_t\notin\SExtendedset_t.
\]

Since $\SExtendedset_t$ is the intersection of the $\del$-relaxed feedback
halfspaces generated earlier in the same batch, there exists some previous
round $j<t$ such that
\[
    \hparam_t\notin\HExtendedset_j.
\]
At round $j$, however, $\hparam_j$ satisfied the corresponding exact
halfspace $\Hset_j$.

If $\fb_j=0$, then
\[
    \ip{\hparam_j}{\uvec_j}\le\price_j,
    \qquad
    \ip{\hparam_t}{\uvec_j}>\price_j+\del,
\]
and hence, using $\norm{\uvec_j}=1$,
\[
    \norm{\hparam_t-\hparam_j}
    \ge
    \ip{\hparam_t-\hparam_j}{\uvec_j}
    >
    \del.
\]
If $\fb_j=1$, then
\[
    \ip{\hparam_j}{\uvec_j}\ge\price_j,
    \qquad
    \ip{\hparam_t}{\uvec_j}<\price_j-\del,
\]
and similarly
\[
    \norm{\hparam_t-\hparam_j}
    >
    \del.
\]
Thus every rejected-LEP reset certifies more than $\del$ variation within the
batch.

The same conclusion holds if a batch is restarted because
$\SExtendedset_t=\varnothing$. 
Suppose that $S_t^+=\emptyset$ at the beginning of round $t$.
If the variation diameter of the completed batch were at most $\delta$,
then for every earlier round $j$ in that batch, $\theta_{t-1}$ would
satisfy the $\delta$-relaxed feedback halfspace generated at round $j$.
It also satisfies the relaxed halfspace generated at round $t-1$.
Hence $\theta_{t-1}\in S_t^+$, contradicting $S_t^+=\emptyset$.

Since different reset batches are disjoint intervals and the total variation
budget satisfies
\[
    \TV(\Theta)\le\TVUB,
\]
there can be at most $\TVUB/\del$ non-window resets. The hard window cap
contributes at most $\TT/\Wsize$ additional batches, plus one final incomplete
batch.
\end{proof}

\subsection{Proof of \Cref{thm:single-scale-contextual}}

\begin{proof}
Let $\batch_i$ denote the $i$-th batch and define its variation diameter by
\[
    \Delta_i
    :=
    \max_{j,k\in\batch_i}
    \norm{\hparam_j-\hparam_k}.
\]
We split the batches into \emph{good} batches, for which
$\Delta_i\le\del$, and \emph{bad} batches, for which
$\Delta_i>\del$.

First consider a good batch. For every round $t$ in this batch and every
previous round $j<t$ in the same batch, the exact feedback halfspace generated
at $j$ contains $\hparam_j$. Since
\[
    \norm{\hparam_t-\hparam_j}
    \le
    \Delta_i
    \le
    \del,
\]
the current parameter $\hparam_t$ satisfies the corresponding
$\del$-relaxed halfspace. Therefore
\[
    \hparam_t\in\SExtendedset_t
\]
throughout the batch meaning that every LEP round in a good batch is accepted
\[
\price_t \le \ip{\theta_t}{\uvec_t}
\]

Now consider an LEP exploitation round $t$ in this good batch. Since the
round is not exploratory,
\[
    \width(\mathcal E_t,\uvec_t)\le\eta.
\]
Because
\[
    \SExtendedset_t\subseteq\mathcal E_t,
\]
we have
\[
    \width(\SExtendedset_t,\uvec_t)
    \le
    \width(\mathcal E_t,\uvec_t)
    \le
    \eta.
\]
The LEP price is
\[
    \price_t
    \ge
    \min_{x\in\SExtendedset_t}
    \ip{x}{\uvec_t}.
\]
Since $\hparam_t\in\SExtendedset_t$ and the realized valuation is
nonnegative, the clipped LEP price cannot exceed the current valuation, and hence
\[
\begin{aligned}
    \ip{\hparam_t}{\uvec_t}-\price_t
    &\le
    \max_{x\in\SExtendedset_t}\ip{x}{\uvec_t}
    -
    \min_{x\in\SExtendedset_t}\ip{x}{\uvec_t} \\
    &=
    \width(\SExtendedset_t,\uvec_t)
    \le
    \eta.
\end{aligned}
\]
Thus every accepted LEP round in a good batch incurs regret at most $\eta$.

Since exploration prices are clipped to $[0,1]$, exploration rounds incur regret at most $1$. By
\Cref{lem:single-scale-exploration}, each batch contains only
$\tO(d^2)$ exploration rounds. Hence every good batch $\batch_i$ satisfies
\[
    \bReg_i
    \le
    O(d^2
        \log\frac{20R(d+1)}{\eta})
    +
    |\batch_i|\eta.
\]

For a bad batch, we use the crude bound
\[
    \bReg_i
    \le
    |\batch_i|
    \le
    \Wsize.
\]
Since every bad batch satisfies $\Delta_i>\del$ and the batches are disjoint,
the number of bad batches is at most $\TVUB/\del$.

By \Cref{lem:single-scale-reset}, the total number of batches is at most
\[
    r
    \le
    1+\frac{\TT}{\Wsize}+\frac{\TVUB}{\del}.
\]
Therefore
\[
    \cRT
    \le
    O(d^2
        \log\frac{20R(d+1)}{\eta})
    \left(
        \frac{\TT}{\Wsize}
        +
        \frac{\TVUB}{\del}
        +
        1
    \right)
    +
    \TT\eta
    +
    \Wsize\frac{\TVUB}{\del}.
\]
For notational simplicity, we ignore the ceiling in $\Wsize$ in the
tuning calculation below. Since the unrounded quantity is at least one,
this changes the bounds only by an absolute constant factor.
With
\[
    \Wsize
    =
    \min \{\TT,d^{1/2}\left(\frac{\TT}{\TVUB}\right)^{1/3}\},
    \qquad
    \del
    =
    \max \{ \frac{1}{\TT},
    d^{1/2}\left(\frac{\TVUB}{\TT}\right)^{1/3}\},
    \qquad
    \eta
    =
    4d\del,
\]

If $\eta\ge 20R(d+1)$, the claimed tuned bound follows from the
trivial regret bound $\Reg[\ALG]\le T$, since
\[
    d^{3/2}\TVUB^{1/3}\TT^{2/3}
    =
    \frac{\eta\TT}{4}
    \ge
    5R(d+1)\TT.
\]
Hence it suffices to consider $\eta<20R(d+1)$, where
\Cref{lem:single-scale-exploration} applies.

One may introduce 
\[
Q:=\TVUB^{1/3}\TT^{2/3},
\]
and see that 
\[
\frac{\TT}{\Wsize}\le 1+ \frac{Q}{\sqrt{\dimension}},
\qquad
\frac{\TVUB}{\delta}\le \frac{Q}{\sqrt{\dimension}},
\]
Then 
\[
\TT\eta\le 4\dimension+4\dimension^{3/2}Q, \qquad
\Wsize\frac{\TVUB}{\delta}\le Q,
\]
Overall proving
\begin{align}\label{eq: singlescale-regret-tuned}
    \cRT
    \le
    O\!\left(
        (d^2 + d^{3/2}\TVUB^{1/3}\TT^{2/3})\log (\TT)
    \right),
\end{align}

as claimed.
\end{proof}

We next show that the selected single-scale tuning can incur
$\Omega(\TVUB^{1/3}\TT^{2/3})$ regret even on a stationary instance.

\begin{proposition}[Lower bound for the selected single-scale tuning]
\label{prop:single-scale-tuning-lower-bound}
Consider the single-scale shallow pricing algorithm in dimension
$\dimension=2$, with parameters
\[
    \del=
    \max\left\{
        \frac1{\TT},
        \sqrt{2}\left(\frac{\TVUB}{\TT}\right)^{1/3}
    \right\},
    \qquad
    \Wsize=
    \min\left\{
        \TT,
        \left\lceil
        \sqrt{2}\left(\frac{\TT}{\TVUB}\right)^{1/3}
        \right\rceil
    \right\},
    \qquad
    \eta=8\del,
\]
where $\TVUB\in(0,\TT]$.
At each restart, initialize the relaxed feasible set to
\[
    \Sset_0
    =
    \{x\in\Rspace^2:x\ge0,\ \|x\|\le1\}
\]
and the ellipsoid to
\[
    \Eset_0
    =
    \{x\in\Rspace^2:\|x\|\le1\}.
\]
Suppose each batch contains at most $\Wsize$ rounds.

There exists a stationary instance with identical contexts on which the
algorithm incurs dynamic regret of at least
\[
    \frac{3}{8\sqrt{2}}
    \TVUB^{1/3}\TT^{2/3}.
\]
The same lower bound holds when exploration prices are clipped to $[0,1]$.
\end{proposition}

\begin{proof}
For every round $t\in[\TT]$, set
\[
    \uvec_t=\uvec
    =
    \left(\frac1{\sqrt2},\frac1{\sqrt2}\right),
    \qquad
    \theta_t=\theta
    =
    \left(\frac{3}{4\sqrt2},\frac{3}{4\sqrt2}\right).
\]
These contexts and parameters satisfy the model assumptions, and every
buyer has value
\[
    v_t=\ip{\theta}{\uvec}=\frac34.
\]
Moreover, the parameter sequence is stationary, so its total variation
is zero and hence does not exceed $\TVUB$.

Consider the first round of any batch. If the algorithm explores, it
posts the midpoint of the projection of $\Eset_0$ onto $\uvec$, namely
\[
    \ip{0}{\uvec}=0.
\]
This price is accepted and incurs regret $3/4$. If the algorithm instead
exploits, it posts the lower-envelope price of $\Sset_0$ in direction
$\uvec$, which is
\[
    \max\left\{
        0,
        \min_{x\in\Sset_0}\ip{x}{\uvec}
    \right\}
    =0.
\]
This also incurs regret $3/4$. Thus, regardless of whether the algorithm
initially explores or exploits, the first round of each batch incurs
$3/4$ regret. Clipping exploration prices to $[0,1]$ does not change
this conclusion, since the initial exploration price is $0$.

Let $r$ denote the number of batches. Since each batch contains at most
$\Wsize$ rounds,
\[
    r\ge\frac{\TT}{\Wsize}.
\]
All other rounds incur nonnegative regret, so
\[
    \Reg[\ALG]
    \ge
    \frac{3r}{4}
    \ge
    \frac{3\TT}{4\Wsize}.
\]
Since $\TVUB\le\TT$,
\[
    \sqrt{2}\left(\frac{\TT}{\TVUB}\right)^{1/3}
    \ge \sqrt{2}>1.
\]
Therefore,
\[
    \Wsize
    \le
    \left\lceil
        \sqrt{2}\left(\frac{\TT}{\TVUB}\right)^{1/3}
    \right\rceil
    \le
    2\sqrt{2}\left(\frac{\TT}{\TVUB}\right)^{1/3}.
\]
Consequently,
\[
    \Reg[\ALG]
    \ge
    \frac{3}{8\sqrt{2}}
    \TVUB^{1/3}\TT^{2/3},
\]
as claimed.

This lower bound matches the
$\TVUB^{1/3}\TT^{2/3}$ dependence of the stated upper bound for the
exact tuning of Algorithm~\ref{alg:single-scale-contextual}, up to
logarithmic factors when the dimension is fixed.

In particular, when $\TVUB=1$, this tuning incurs
$\Omega(\TT^{2/3})$ regret even on a stationary instance. The result
concerns the selected parameter choice; it does not rule out improved
guarantees under a different tuning of the single-scale algorithm.
\end{proof}

\section{Multiscale Contextual Analysis}
\label{app:multiscale-contextual}

This appendix gives the complete analysis of
\Cref{alg:multiscale-contextual}. We first establish the nesting and
containment properties of the relaxed feasible sets, then bound the exploration
complexity and LEP regret, and finally combine these ingredients to prove
\Cref{thm:multiscale-contextual}.

\subsection{Structural properties of the relaxed sets}
\label{app:multiscale-structure}

Recall the dyadic scale set
\[
    \Lambda=\{\del,2\del,\ldots,2^L\del\},
\]
where $2^L\delta\ge R$.

\begin{lemma}[Nesting across scales]
\label[lemma]{lem:multiscale-nesting}
If $\lambda\le\lambda'$, then for every round $t$ within a batch,
\[
    \Sset_t^{(\lambda)}
    \subseteq
    \Sset_t^{(\lambda')}.
\]
Consequently, the active set
\[
    \Lambda_t
    =
    \{\lambda\in\Lambda:\Sset_t^{(\lambda)}\neq\varnothing\}
\]
is a suffix of $\Lambda$.
\end{lemma}

\begin{proof}
For each observed feedback sign, increasing the relaxation radius only
enlarges the corresponding halfspace:
\[
    H_j^{(\lambda)}
    \subseteq
    H_j^{(\lambda')}
    \qquad
    \text{whenever }\lambda\le\lambda'.
\]
Intersecting these halfspaces over previous rounds preserves the inclusion.
The suffix property follows immediately.
\end{proof}

\begin{lemma}[Existence of an active scale]
\label[lemma]{lem:multiscale-active-nonempty}
For every round $t$ within a batch,
\[
    \Lambda_t\neq\varnothing.
\]
In particular,
\[
    \lambda_t^{\min}
    =
    \min\Lambda_t
\]
is well-defined.
\end{lemma}

\begin{proof}
Let
\[
    \lambda_{\max}=2^L\del\ge R.
\]
Fix any current round $t$ and any previous round $j<t$ in the same batch.
The parameter $\hparam_j$ satisfies the exact feedback halfspace generated at
round $j$. Since both $\hparam_j$ and $\hparam_t$ lie in $\Sset_0$,
\[
    \norm{\hparam_t-\hparam_j}
    \le
    \operatorname{diam}(\Sset_0)
    \le
    R
    \le
    \lambda_{\max}.
\]
Using $\norm{\uvec_j}=1$, the current parameter $\hparam_t$ therefore
satisfies the $\lambda_{\max}$-relaxed version of the feedback halfspace from
round $j$. This holds for every previous round, so
\[
    \hparam_t
    \in
    \Sset_t^{(\lambda_{\max})}.
\]
Hence the largest scale is active.
\end{proof}

\subsection{Containment at the retrospective right scale}
\label{app:multiscale-containment}

For a batch $\batch_\phase$, define its variation diameter by
\[
    \Delta_\phase
    :=
    \max_{j,k\in\batch_\phase}
    \norm{\hparam_j-\hparam_k},
\]
and define the retrospective right scale
\[
    \lambda_\phase
    :=
    \min\{\lambda\in\Lambda:\lambda\ge\Delta_\phase\}.
\]
This scale exists because
\[
    \Delta_\phase
    \le
    \operatorname{diam}(\Sset_0)
    \le
    R.
\]

\begin{lemma}[Multiscale containment]
\label[lemma]{lem:multiscale-containment}
For every $t\in\batch_\phase$ and every $s\in\batch_\phase$ with $s\le t$,
\[
    \hparam_s
    \in
    \Sset_t^{(\lambda_\phase)}.
\]
Moreover,
\[
    \lambda_\phase
    \le
    2\max\{\Delta_\phase,\del\}.
\]
\end{lemma}

\begin{proof}
Fix $s\le t$ and a previous round $j<t$ in the same batch. If
$\fb_j=0$, then
\[
    \ip{\hparam_j}{\uvec_j}
    \le
    \price_j.
\]
Because
\[
    \norm{\hparam_s-\hparam_j}
    \le
    \Delta_\phase
    \le
    \lambda_\phase,
\]
we obtain
\[
    \ip{\hparam_s}{\uvec_j}
    \le
    \ip{\hparam_j}{\uvec_j}
    +
    \norm{\hparam_s-\hparam_j}
    \le
    \price_j+\lambda_\phase.
\]
Thus $\hparam_s$ satisfies $H_j^{(\lambda_\phase)}$.

The case $\fb_j=1$ is symmetric:
\[
    \ip{\hparam_s}{\uvec_j}
    \ge
    \ip{\hparam_j}{\uvec_j}
    -
    \norm{\hparam_s-\hparam_j}
    \ge
    \price_j-\lambda_\phase.
\]
Hence $\hparam_s$ satisfies every relaxed halfspace defining
$\Sset_t^{(\lambda_\phase)}$.

The dyadic definition of $\lambda_\phase$ gives
\[
    \lambda_\phase<2\Delta_\phase
    \qquad
    \text{whenever }\Delta_\phase\ge\del,
\]
while
\[
    \lambda_\phase=\del
    \qquad
    \text{when }\Delta_\phase<\del.
\]
Therefore
\[
    \lambda_\phase
    \le
    2\max\{\Delta_\phase,\del\}.
\]
\end{proof}

\begin{corollary}[The right scale is active]
\label[corollary]{cor:right-scale-active}
For every $t\in\batch_\phase$,
\[
    \lambda_\phase\in\Lambda_t
    \qquad\text{and hence}\qquad
    \lambda_t^{\min}\le\lambda_\phase.
\]
\end{corollary}

\begin{proof}
\Cref{lem:multiscale-containment} gives
\[
    \hparam_t
    \in
    \Sset_t^{(\lambda_\phase)},
\]
so the set is nonempty and $\lambda_\phase$ is active. The second claim
follows from the definition of $\lambda_t^{\min}$.
\end{proof}

\subsection{Exploration complexity}
\label{app:multiscale-exploration}

\begin{lemma}[Multiscale exploration bound]
\label[lemma]{lem:multiscale-exploration}
Suppose $\TVUB\le \TT/\dimension$. Within any fixed batch, \Cref{alg:multiscale-contextual} performs at
most
\[
    20d^2(L+1)
    \ln\!\left(
        \frac{5R(d+1)}{d\del}
    \right)
\]
exploration rounds.
\end{lemma}

\begin{proof}
Fix a scale $\lambda\in\Lambda$. The algorithm selects this scale only when it
is active and
\[
    \width(\mathcal E_t^{(\lambda)},\uvec_t)
    >
    4d\lambda.
\]
Whenever it is selected, only
$\mathcal E_t^{(\lambda)}$ is shallow-updated using the standard
$\lambda$-shifted shallow cut associated with the unclipped midpoint proposal.
By \Cref{app:technical-background}, this update remains a valid outer approximation when the posted
exploration price is clipped at $[0,1]$.
Since $\TVUB\le \TT/\dimension$, we have $\del\le1<R$.
By the minimality of $\scalelevels$, every
$\spscale\in\spscaleSet$ satisfies $\spscale<2R$. Hence
\[
    4\dimension\spscale
    <
    8\dimension R
    <
    20R(\dimension+1),
\]
so \Cref{lem:generic-shallow-exploration} applies with
$\eta=4\dimension\spscale,$ bounding the number of
times this scale can be selected by
\[
    20d^2
    \ln\!\left(
        \frac{20R(d+1)}{4d\lambda}
    \right)
    \le
    20d^2
    \ln\!\left(
        \frac{5R(d+1)}{d\del}
    \right),
\]
where we used $\lambda\ge\del$. Summing over the $L+1$ dyadic scales proves
the claim.
\end{proof}

\subsection{Accepted-LEP regret}
\label{app:multiscale-lep}

\begin{lemma}[Accepted LEP regret at the right scale]
\label[lemma]{lem:multiscale-lep-regret}
Consider an accepted LEP round $t$ in batch $\batch_\phase$. Then
\[
    \ip{\hparam_t}{\uvec_t}-\price_t
    \le
    8d\max\{\Delta_\phase,\del\}.
\]
\end{lemma}

\begin{proof}
By \Cref{cor:right-scale-active}, $\lambda_\phase$ is active and
\[
    \lambda_t^{\min}
    \le
    \lambda_\phase.
\]
Since round $t$ is an LEP round, no active scale violates its width threshold.
In particular,
\[
    \width(
        \mathcal E_t^{(\lambda_\phase)},
        \uvec_t
    )
    \le
    4d\lambda_\phase.
\]
Using
\[
    \Sset_t^{(\lambda_\phase)}
    \subseteq
    \mathcal E_t^{(\lambda_\phase)},
\]
we also have
\[
    \width(
        \Sset_t^{(\lambda_\phase)},
        \uvec_t
    )
    \le
    4d\lambda_\phase.
\]

By \Cref{lem:multiscale-nesting},
\[
    \Sset_t^{(\lambda_t^{\min})}
    \subseteq
    \Sset_t^{(\lambda_i)}.
\]
Therefore, by monotonicity of $z\mapsto\max\{0,z\}$,
\[
\begin{aligned}
    \price_t
    &=
    \operatorname{LEP}
    \bigl(\Sset_t^{(\lambda_t^{\min})},\uvec_t\bigr)
    \\
    &\ge
    \operatorname{LEP}
    \bigl(\Sset_t^{(\lambda_i)},\uvec_t\bigr)
    \\
    &\ge
    \min_{x\in\Sset_t^{(\lambda_i)}}
    \ip{x}{\uvec_t}.
\end{aligned}
\]
\Cref{lem:multiscale-containment} gives
\[
    \hparam_t
    \in
    \Sset_t^{(\lambda_\phase)}.
\]
Since the LEP round is accepted,
\[
\begin{aligned}
    \ip{\hparam_t}{\uvec_t}-\price_t
    &\le
    \ip{\hparam_t}{\uvec_t}
    -
    \min_{x\in\Sset_t^{(\lambda_\phase)}}
    \ip{x}{\uvec_t}
    \\
    &\le
    \width(
        \Sset_t^{(\lambda_\phase)},
        \uvec_t
    )
    \\
    &\le
    4d\lambda_\phase
    \\
    &\le
    8d\max\{\Delta_\phase,\del\}.
\end{aligned}
\]
\end{proof}

\subsection{Reset count}
\label{app:multiscale-resets}

\begin{lemma}[Multiscale reset count]
\label[lemma]{lem:multiscale-reset-count}
The number of LEP-rejection resets in
\Cref{alg:multiscale-contextual} is at most
\[
    \frac{\TVUB}{\del}.
\]
Consequently, the total number of batches satisfies
\[
    \Tphase
    \le
    1+\frac{\TT}{\Wsize}+\frac{\TVUB}{\del}.
\]
\end{lemma}

\begin{proof}
Consider a rejected LEP round $t$. Rejection means
\[
    \ip{\hparam_t}{\uvec_t}<\price_t.
\]
Since the realized valuation is nonnegative, we must have $\price_t>0$.
Hence the clipping at zero is inactive, and therefore
\[
    \price_t
    =
    \operatorname{LEP}
    \bigl(\Sset_t^{(\lambda_t^{\min})},\uvec_t\bigr)
    =
    \min_{x\in\Sset_t^{(\lambda_t^{\min})}}
    \ip{x}{\uvec_t}.
\]
It follows that
\[
    \hparam_t\notin\Sset_t^{(\lambda_t^{\min})}.
\]

Hence $\hparam_t$ violates some earlier
$\lambda_t^{\min}$-relaxed feedback halfspace generated at a round $j$ in
the same batch, while $\hparam_j$ satisfied the corresponding exact
halfspace.

If $\fb_j=0$, then
\[
    \ip{\hparam_j}{\uvec_j}
    \le
    \price_j,
    \qquad
    \ip{\hparam_t}{\uvec_j}
    >
    \price_j+\lambda_t^{\min},
\]
so
\[
    \norm{\hparam_t-\hparam_j}
    >
    \lambda_t^{\min}
    \ge
    \del.
\]
The case $\fb_j=1$ is symmetric. Thus every LEP-rejection reset certifies
more than $\del$ variation within its batch.

Since reset batches are disjoint and the total variation budget is at most
$\TVUB$, there can be at most $\TVUB/\del$ such resets. Window-size resets
contribute at most $\TT/\Wsize$ additional batches, plus one final batch.
\end{proof}

\subsection{Proof of \Cref{thm:multiscale-contextual}}
\label{app:proof-multiscale}

\begin{proof}

If $\TVUB>\TT/\dimension$, the claimed bound follows from the trivial regret bound
$\Reg[\ALG]\le \TT$, since
\[
    \dimension^{3/2}\sqrt{\TT\TVUB}\ge \TT.
\]
Hence, for the remainder of the proof, it suffices to consider
$\TVUB\le \TT/\dimension$.
Let
\[
    \batch_1,\ldots,\batch_{\Tphase}
\]
be the batches. Since the batches are disjoint intervals,
\[
    \sum_{\phase=1}^{\Tphase}\Delta_\phase
    \le
    \TV(\Theta)
    \le
    \TVUB.
\]
We decompose regret into exploration rounds, accepted LEP rounds, and rejected
LEP rounds.

Since exploration prices are nonnegative and realized valuations lie in
$[0,1]$, each exploration round incurs regret at most $1$. Then by \Cref{lem:multiscale-exploration}, exploration contributes at most
\[
    20d^2(L+1)
    \ln\!\left(
        \frac{5R(d+1)}{d\del}
    \right)
    \Tphase.
\]

For accepted LEP rounds in batch $\batch_\phase$,
\Cref{lem:multiscale-lep-regret} gives per-round regret at most
\[
    8d\max\{\Delta_\phase,\del\}
    \le
    8d(\Delta_\phase+\del).
\]
Since
\[
    |\batch_\phase|
    \le
    \Wsize,
\]
summing over batches yields
\[
\begin{aligned}
    \text{accepted-LEP regret}
    &\le
    8d\Wsize
    \sum_{\phase=1}^{\Tphase}\Delta_\phase
    +
    8d\del
    \sum_{\phase=1}^{\Tphase}|\batch_\phase|
    \\
    &\le
    8d\Wsize\TVUB
    +
    8d\del\TT.
\end{aligned}
\]

By \Cref{lem:multiscale-reset-count}, there are at most
$\TVUB/\del$ rejected LEP rounds, each with regret at most $1$. Moreover,
\[
    \Tphase
    \le
    1+\frac{\TT}{\Wsize}+\frac{\TVUB}{\del}.
\]
Therefore
\[
\begin{aligned}
    \cRT
    \le\;&
    20d^2(L+1)
    \ln\!\left(
        \frac{5R(d+1)}{d\del}
    \right)
    \left(
        1+\frac{\TT}{\Wsize}+\frac{\TVUB}{\del}
    \right)
    \\
    &\quad+
    8d\Wsize\TVUB
    +
    8d\del\TT
    +
    \frac{\TVUB}{\del}.
\end{aligned}
\]
For notational simplicity, in the tuning calculation below we ignore the
ceiling in the definition of $\Wsize$. This affects the bounds only by an
absolute constant factor: for $x\ge 1$,
\[
    x\le \lceil x\rceil\le 2x,
\]
and the additional truncation by $\TT$ preserves this constant-factor
relation. Thus replacing
$\min\{\TT,\lceil \sqrt{\dimension\TT/\TVUB}\rceil\}$ by
$\min\{\TT,\sqrt{\dimension\TT/\TVUB}\}$ does not affect the stated big-$O$ regret bound.
Choose
\[
    \Wsize
    =
    \min \{ \TT, \sqrt{\frac{d\TT}{\TVUB}} \},
    \qquad
    \del
    =
    \max \{ \frac{1}{\TT},\sqrt{\frac{d\TVUB}{\TT}}\}.
\]
Then
\[
    \frac{\TT}{\Wsize}
    \le 1+
    \sqrt{\frac{\TVUB\TT}{d}},
    \qquad 
    \frac{\TVUB}{\del} \le \sqrt{\frac{\TVUB\TT}{d}},
\]
and
\[
    d\Wsize\TVUB
    \le
    d^{3/2}\sqrt{\TVUB\TT},
    \qquad
    d\del\TT
    \le
    d+ d^{3/2}\sqrt{\TVUB\TT}.
\]
Finally,
\[
    L+1
    =
    O\!\left(
        1+\log\frac{R}{\del}
    \right),
\]
Hence
\[
    \cRT
    =
    O\!\left(
        (d^2
        +
        d^{3/2}\sqrt{\TVUB\TT})\log^2(\TT)
    \right).
\]

\end{proof} 
\end{document}